\documentclass[11pt]{article}

\usepackage{fullpage}
\usepackage[utf8]{inputenc} 
\usepackage[T1]{fontenc}    
\usepackage{hyperref}       
\usepackage{url}            
\usepackage{booktabs}       
\usepackage{amsmath,amsfonts, amsthm}       
\usepackage{nicefrac}       
\usepackage{microtype}      
\usepackage{xcolor}         
\usepackage{caption}
\usepackage{subcaption}
\usepackage{algorithmic}

\usepackage{algorithm}

\usepackage{xcolor}
\usepackage{tikz}
\usetikzlibrary{shapes,arrows}
\tikzstyle{decision} = [diamond, draw, fill=blue!20, 
    text width=4.5em, text badly centered, node distance=3cm, inner sep=0pt]
\tikzstyle{block} = [rectangle, draw, fill=blue!20, 
    text width=5em, text centered, rounded corners, minimum height=4em]
\tikzstyle{line} = [draw, -latex']
\tikzstyle{cloud} = [draw, ellipse,fill=red!20, node distance=3cm,
    minimum height=2em]
\tikzset{every picture/.style={line width=0.75pt}}

\usepackage{dsfont}
\usepackage[short]{optidef}
\usepackage{comment}
\usepackage{thm-restate}
\usepackage{authblk}

\newcommand{\Ps}{\mathcal{P}}

\newcommand{\X}{\mathcal{X}}
\newcommand{\Y}{\mathcal{Y}}
\newcommand{\Z}{\mathcal{Z}}
\newcommand{\1}{\mathds{1}}
\newcommand{\Hs}{\mathcal{H}}
\newcommand{\E}{\mathbb{E}}
\newcommand{\R}{\mathbb{R}}

\newcommand\scalemath[2]{\scalebox{#1}{\mbox{\ensuremath{\displaystyle #2}}}}

\newtheorem{definition}{Definition}

\newtheorem{thm}{Theorem}

\newtheorem*{remark}{Remark}
\DeclareMathOperator*{\argmin}{argmin}

\title{Balanced Filtering via Disclosure-Controlled Proxies}

\author[1]{Siqi Deng*}
\author[1,2]{Emily Diana$^\ddagger$}
\author[1,2]{Michael Kearns}
\author[1,2]{Aaron Roth}
\affil[1]{Amazon AWS AI}
\affil[2]{University of Pennsylvania}

\date{}

\begin{document}
\maketitle

\def\thefootnote{$\ddagger$}\footnotetext{Work conducted during an internship at Amazon}
\def\thefootnote{*}\footnotetext{Corresponding author}

\begin{abstract}
We study the problem of collecting a cohort or set that is \textit{balanced} with respect to sensitive groups when group membership is unavailable or prohibited from use at deployment time. Specifically, our deployment-time collection mechanism does not reveal significantly more about the group membership of any individual sample than can be ascertained from base rates alone. To do this, we study a learner that can use a small set of labeled data to train a proxy function that can later be used for this filtering or selection task. We then associate the range of the proxy function with sampling probabilities; given a new example, we classify it using our proxy function and then select it with probability corresponding to its proxy classification.  Importantly, we require that the proxy classification does not reveal significantly more information about the sensitive group membership of any individual example compared to population base rates alone (i.e., the level of disclosure should be controlled) and show that we can find such a proxy in a sample- and oracle-efficient manner. Finally, we experimentally evaluate our algorithm and analyze its generalization properties.
\end{abstract}
\section{Introduction}
\label{sec:introduction}
There are a variety of situations in which we would like to select a cohort or set that is \emph{balanced} or \emph{representative} (having an approximately equal number of samples from different groups) with respect to race, sex, or other sensitive attributes --- but, we cannot explicitly select based on these attributes. This could be because the attributes are sensitive so were never collected, they could be redacted from the information we see, they could be too resource intensive to collect, or selecting based on these attributes could be illegal.

Consider the context of college admissions. Out of many qualified applicants, a college may prioritize racial diversity when deciding upon the final cohort to admit. However, in the United States Supreme Court decision for \textit{Students for Fair Admissions, Inc. v. President and Fellows of Harvard College}, it was determined that ``Harvard's and UNC's [race-conscious] admissions programs violate the Equal Protection Clause of the Fourteenth Amendment'' \cite{scHarvard}. How might a  college select a racially diverse cohort with affirmative action prohibited?

Our approach is based on training a proxy classifier --- in the form of a decision-tree --- with the following properties:
(1) The set of points classified at each leaf should not be strongly correlated with the protected attribute (the \textit{disclosure-control} part) and
(2) The set of distributions on the protected attributes induced at each leaf should be such that the uniform distribution on protected attributes is in their convex hull (the \textit{balancing} part). The second condition allows us to assign sampling probabilities to the leaves such that if we accept each example with probability corresponding to its proxy classification, in expectation the selected cohort will be balanced with respect to the protected attribute.\footnote{A natural first approach is to add noise to the predictor for the protected attribute, against which we compare. However, we are motivated by the need to have strategies that never involve training a classifier for the protected attribute, especially if it could be used outside of the intended system.}
\subsection{Related Work}
The proxy problem is a subject of ongoing debate in the philosophy of science and causal inference literatures (e.g. \cite{pittphilsci17169,8086cc7d-8c04-3b85-8195-ae76aeea22a5, tchetgen2020introduction, Miao2016IdentifyingCE, 2deb2b11-39b5-327e-ade8-8ba30ec0e1fa,qiu2023doubly,shi2023theory,doi:10.1080/01621459.2023.2191817}), and our work engages with this literature methodologically -- we do not believe that it is our role to take a philosophical or legal stance but rather to broaden the set of available tools. Using proxy variables for sensitive attributes in settings where diversity or equity is a concern has been standard practice, yet in many cases, existing features are chosen for the proxies (such as surname, first name, or geographic location \cite{surname, firstName, Zhang2016AssessingFL}). Rather than using an existing feature as a proxy, we propose deliberately \textit{constructing} a proxy. Several works take this perspective -- in \cite{proxies}, for example, the authors produce a proxy that can be used during training to build a fair model downstream. But often, proxies for protected attributes are explicitly intended to be good predictors for those attributes; it is not clear that using an accurate ``race predictor'' is an acceptable solution to making decisions in which race should not be used (and is often explicitly prohibited). Our primary point of departure is that we train a model to make classifications that are \textit{minimally correlated} with the protected attribute.

While our intended use cases are primarily curation or cohort selection, one may also use our method for collecting balanced data sets for machine learning applications. However, we recommend caution in these scenarios, as our approach does not give guarantees about the level of distortion of the final filtered data set. In order to provide comparisons to existing empirical techniques, however, we do measure our approach against a common data pre-processing technique, SMOTE (Synthetic Minority Oversampling Technique) \cite{smote}. Other re-sampling methods for data balancing include ADASYN \cite{He2008ADASYNAS}, MIXUP \cite{mixup}, SMOTE adaptations \cite{Nguyen2009BorderlineOF, Batista2003BalancingTD, 10.1145/1007730.1007735, Douzas_2018, Han2005BorderlineSMOTEAN, Menardi2012TrainingAA}) and cluster-based approaches that under-sample disproportionately represented classes \cite{CondensedNearestNeighbor, Wilson1972AsymptoticPO, 4309523, kNN, tomek2}. In the causal literature, propensity score re-weighting \cite{propensityScoreReweighting} is also a popular approach to account for group size differences. Each of these techniques, however, requires access to the sensitive attribute. Our approach's primary point of departure is that we do not use the sensitive attribute---or a direct prediction or imputation of it---at the final collection time when we are deploying our method.
\subsection{Limitations and Discussion}
Our contributions are twofold. First, for situations where balance is desired but disclosure is not a concern, we introduce a sampling scheme optimized to collect a balanced cohort. Second, for when disclosure is a concern, we present a method to produce a proxy function for the balanced selection task that is \textit{guaranteed} not to be too disclosive. Below, we discuss important considerations having to do with appropriate usage of our methodology, limitations, and areas for expansion.
\begin{itemize}
\item \textbf{The sensitive attribute is still used to inform the proxy, and our approach relies on accessing a small sample of data with this attribute:}
The proxy training algorithm we propose is not blind to the sensitive attributes, which it must access during \textit{training}. Rather, the proxy does not use these attributes at the time of \textit{deployment}.\footnote{It would be impossible to give an algorithm making no use of the protected attribute during deployment or training and yet promising any sort of  balance --- it would have to  behave identically on any distributions with the same marginals over non-protected attributes, even if they differed on the protected attribute.} We emphasize that there is no contradiction between (1) being able to obtain (once) a small data set labeled with sensitive attributes and then using it to train a classification algorithm (in this case, our proxy model) and (2) having the inability to collect or use sensitive attributes when gathering the bulk of one's data. This is especially true when the final selection criterion is not closely correlated with the sensitive attribute, which is one of our primary objectives. In the algorithmic fairness literature in particular, there is a substantial and growing body of work on learning classifiers that satisfy fairness constraints by sensitive attributes but that do not use these attributes at test time (e.g. \cite{reductions, laftr, 10.1145/3461702.3462629, Prost2019TowardAB}). \textit{These methods still require access to the attributes at training time.} 

The distinction between using sensitive attributes at train versus test time is essential. In certain financial applications in the United States, using race or gender at test time (i.e., when making lending decisions) is illegal. But it is not illegal to use these attributes at training time to audit models for statistical bias and to remove it if found. The distinction in our case is similar: We use these attributes to find a statistical selection criterion but do not use sensitive attributes of individuals to make selection decisions about them.
\item \textbf{The filtered cohort or data set will likely exhibit within-group distortion:}
This is an important consideration that should be taken into account when using our method. Our theoretical guarantees provide bounds on the level of balance and disclosure when measured with respect to the sensitive attributes, but they do not guarantee that the distribution over selected individuals matches that of the true population. In fact, this is perhaps a necessary effect of our process and in many cases may be natural. For example, in the context of college admissions or interview selection, a university or firm is intentionally selecting a pool that is \textit{not} representative of the base population. The use cases for which this quality may create the greatest challenge is in curating data sets for training machine learning models. While our method can improve \textit{representation} in data sets, it will not necessarily lead to improvements in downstream fairness of models trained on the balanced data. We provide a detailed analysis of this in Appendix~\ref{sec:app_exp}.\footnote{One note, however, is that our method does allow for balancing multiple attributes at a time -- therefore, one could ask for a cohort that has the same number of positive and negative examples in each group.}
\item \textbf{Affirmative Action and Legal Challenges:}
We do not propose our method as a way to circumvent the intent of legislation, nor do we make claims regarding the legal or moral appropriateness of its usage in any particular affirmative action setting. Rather, we view it as a tool that can be used, when permitted legally, in settings where diversity or balance is desired but when the sensitive attribute can or should be used only minimally. For example, in the college admissions example, it is also undesirable to use explicit race-based predictors in lieu of observing the sensitive attribute. However, students can include race considerations in their admissions essays, which admissions officers see. Given the stakes of college admissions and the strategic behavior of both sides (applicants and schools), it is very likely that an ad-hoc system will still develop to indirectly make use of racial information for the sake of diversity. In a situation such as this, our method provides a controlled way to achieve such desiderata without unintentionally revealing too much information or making use of highly correlated proxies.
\end{itemize}
\section{Model and Preliminaries} \label{sec:preliminaries}
Let $\Omega = \X \times \Y \times \Z$ be an arbitrary data domain and $\Ps$ be the probability distribution over $\Omega$. $\Ps_x$ will refer to the marginal distribution over $\X$, $\Ps_z$ will refer to the marginal distribution over $\Z$, and $\Ps_{z|x}$ will refer to the conditional distribution over $\Z|\X$. Each data point is a triplet $\omega = (x,y,z) \in \X \times \Y \times \Z$, where $x \in \X$ is the non-sensitive feature vector and $y \in \Y = \{0,1\}$ is the label. The label $y$ is not required in training or applying our filtering method, but it will be used in the analysis of downstream fairness effects of the filtering process provided in Appendix~\ref{sec:app_exp}. In the paper body, therefore, we omit it for clarity.

Unless otherwise specified, we take group membership as disjoint such that $z \in \Z = [K]$ is an integer indicating sensitive group membership, but our framework can easily be extended to the case where group membership need not be disjoint. We consider the uniform distribution, $U$, to be our target distribution over sensitive attributes, where $U=(\frac{1}{K},...,\frac{1}{K})$. We also provide a brief extension to the intersecting case below.

We imagine we can sample unlimited data from $\Ps_x$, but the corresponding value $z$ can only be obtained from self-report, authorized agencies, or human annotation. We use $r_k$ to denote the base rate $\Pr_{z \sim \Ps_z}[z=k]$ in the underlying population distribution. We also assume that we can obtain a limited sample of data $D$ of $n$ samples $\{(x_i,z_i)\}_{i=1}^{n} \subset \Omega$ for which we can observe the true sensitive attribute $z$. We would like to use this sample to collect a much larger set $S \subset \Omega$ such that even if we cannot observe the sensitive attributes, $S$ is balanced with respect to $z$.

Formally, we define a balanced set as follows, where $\Pr_{z \sim S}[z=k]=\frac{1}{|S|} \sum_{i=1}^{|S|} \1_{z_i =k}$ is the empirical distribution of $z$ drawn uniformly from $S$.
\begin{definition}[Balance]
\label{def:balance}
A set $S$ is \textit{balanced} with respect to $K$ disjoint groups $z \in [K]$ if $\Pr_{z \sim S}[z=k] =\frac{1}{K}$ $\forall k$.
\end{definition}
Due to finite sampling, Definition~\ref{def:balance} will rarely be met, even if the underlying \textit{distribution} is uniform over sensitive attributes. Therefore, we also discuss approximate balance:\footnote{We use the $L_2$ norm due to operational reasons, as it allows us to make useful geometric arguments.}
\begin{definition}[$\beta$-Approximate Balance]
\label{def:approximate_balance}
A set $S$ is $\beta$-\textit{approximately balanced} with respect to $K$ disjoint groups $z \in [K]$ if $\lVert (\Pr_{z \sim S}[z=1],...,\Pr_{z \sim S}[z=K]) - (1/K,..,1/K) \rVert _ 2 \leq \beta$.
\end{definition}
Note that $\beta$-approximate balance involves the \textit{distribution} of sensitive groups in $S$: as this distribution deviates farther from the uniform, the imbalance, $\beta$, increases.

In the intersecting groups case, we let the sensitive attribute domain $\mathcal{Z}$ be a binary vector of length $K$, such that $z \in \Z = \{0,1\}^{K}$. We will assume $\Z$ is composed of $G$ group classes $\{Z_i\}_{i=1}^{G}$ (e.g., sex, race, etc), and each group class $Z_i$ has $K_i$ groups. Thus, the vector $z$ will indicate all possible group memberships, where the length of the vector of group memberships is $K = \sum_{i=1}^{G} K_i$. We also replace $U$ with $U_{\text{int}} = (\frac{1}{K_1},...,\frac{1}{K_1},...,\frac{1}{K_G},...,\frac{1}{K_G})$ to indicate the target distribution over intersecting sensitive attributes.

\begin{definition}[Multi-Class Balance]
We will say that a set $S$ is balanced with respect to $K$ intersecting groups if, for any group class $\{Z_i\}_{i=1}^{G}$ composed of groups $\{Z_{i_j}\}_{j=1}^{K_i}$, $\Pr_{z \sim S}[z[Z_{i_j}] = 1]=\frac{1}{K_i}$ $\forall j$.
\end{definition}

\begin{definition}[$\beta$-Approximate Multi-Class Balance]
We say that a set $S$ is $\beta$-approximately balanced with respect to $K$ intersecting groups if 
\[\lVert(\Pr_{z \sim S}[z[Z_{1_1}]=1], \Pr_{z \sim S}[z[Z_{1_2}]=1],...,\Pr_{z \sim S}[z[Z_{G_{K_G - 1}}]=1],\Pr_{z \sim S}[z[Z_{G_{K_G}}]=1]) - U_{\text{int}}\rVert_2 \leq \beta
\]

\end{definition}

\begin{remark}
This definition aims to take into account the fact that for different group classes, there may be a different number of potential groups. For example, there may only be two sex groups but eight income groups. Asking that the representation of each of those categories be one-tenth of the final sample would not make sense. However, our definition does not prevent certain intersections being more represented than others. When the number of groups is small, this can always be dealt with by using the Cartesian product over group classes to enumerate intersectional groups.
\end{remark}

In addition to desiring that our proxy allows us to select an approximately \textit{balanced} cohort, we would also like the classification outcomes of the proxy  not to be overly disclosive. We model this by asking that the posterior distribution on group membership is close to the prior distribution when conditioning on the outcome of the proxy classifier. 
\begin{definition}[$\alpha$-Disclosive Proxy]
\label{def:alpha}
A proxy $g$ is $\alpha$-disclosive (or has disclosure level at most $\alpha$) on set $S$ if, for all sensitive groups $k$ and proxy values $i$, $\lvert \Pr_{z|x \sim S} [z=k|g(x)=i] - \Pr_{z \sim S}[z=k] \rvert \leq \alpha$.
\end{definition}
For any proxy function $g$, we can analyze the distribution of sensitive groups amongst points mapped to each value $k$ in the range of the proxy. Call this conditional distribution $a_k$, let $l$ be the number of unique elements in the range of the proxy, and let $A$ be the $\ell \times K$ matrix whose $k^{th}$ row is $a_k$. Denote the convex hull, defined in Definition~\ref{def:convex_hull}, of the rows of $A$ by $C(A)$. Then, we can add a notion of \textit{balance} into our proxy definition in the following way:
\begin{definition}[$\left(\alpha, \beta \right)$ Proxy]
\label{def:beta}
$g: \mathcal{X} \rightarrow \mathbb{N}$ is an $\left(\alpha, \beta \right)$ proxy if it is $\alpha$-disclosive and $\scalemath{0.9}{\inf_{U' \in C(A)}\lVert U' - U \rVert_2 \leq \beta}$.
\end{definition}
Here, $\inf_{U' \in C(A)}\lVert U' - U \rVert_2$ indicates the Euclidean distance between $U$ and the closest point in $C(A)$. We will slightly abuse notation and refer to this as the distance $\lVert C(A) - U \rVert_2$. The disclosure parameter $\alpha$ controls the amount of additional information the proxy gives about group membership, while the balance parameter $\beta$ quantifies the minimum distance from uniform achievable with any acceptance probabilities for a given proxy. There do exist limitations on how small $\alpha$ can be if we need full balance. With $K$ sensitive groups, the final frequency of each group must be $\frac{1}{K}$ if we desire $\beta=0$: if there is a group with initial frequency $f$, there is no avoiding that $\alpha \geq |f-\frac{1}{K}|$. For some data sets, this unavoidably can be quite large. For example, consider a data set of $\frac{1}{3}$ men and $\frac{2}{3}$ women and assume the proxy $g$ takes values 0 or 1. Of the samples mapped to $g = 0$, $\frac{1}{4}$ are men and $\frac{3}{4}$ are women. Of those mapped to $g=1$, $\frac{1}{2}$ are men and $\frac{1}{2}$ are women. Then, $\alpha = \frac{1}{6}$, because $|\Pr[z=\text{men}|g = 0] - \Pr[z=\text{men}]| = |\frac{1}{3} - \frac{1}{2}| = \frac{1}{6}$. In this example, $\beta$ would be 0, because the convex hull of the conditionals $\Pr[z|g(x)]$ contains $[\frac{1}{2}, \frac{1}{2}]$. Next, we provide definitions for a convex hull and stochastic vector.

\begin{definition}[Convex Hull \cite{boyd2004convex}]
\label{def:convex_hull}
The convex hull of a set of points $S$ in $K$ dimensions is the intersection of all convex sets containing $S$. For $l$ points $s_1,...,s_l$, the convex hull $C$ is  given by the expression:
$C \equiv \{\sum_{i=1}^{l} q_i s_i: q_i \geq 0 \text{ for all } i \text{ and } \sum_{i=1}^{l} q_i = 1\}$
\end{definition}
\begin{definition}[Stochastic Vector \cite{Bill86}]
$v=(v_i)_{i=1}^{\ell}$ is a stochastic vector if $\sum_{i=1}^{\ell} v_i = 1 \text{ and }v_i \geq 0 \; \forall i$.
\end{definition}
Finally, we observe a necessary and sufficient condition for $U$ to be in $C(A)$.
\begin{restatable}[Inclusion in Convex Hull]{lemma}{convexHull}
\label{lem:convex_hull}
Let $A$ be an $l \times K$ matrix and $U$ be a $1 \times K$ vector with $\frac{1}{K}$ in each entry. There exists a stochastic vector $q$ such that $q A = U$ if and only if $U \in C(A)$. 
\end{restatable}

\begin{proof}
Let $a_i$ denote the $i^{th}$ row of $A$. If $U$ is in $C(A)$, then by Definition~\ref{def:convex_hull}, there exists a non-negative vector $q$ such that $\sum_{i=1}^{l}q_i a_i = U$ and $\sum_{i=1}^{l}q_i=1$. Similarly, for any stochastic $q$, $q A \in C(A)$. Then if $q A = U$, $U \in C(A)$.
\end{proof}

Finally, we outline several key results that we will use in the derivations and proofs for our proxy training algorithm. We begin by considering a zero-sum game between two players, a Learner with strategies in $S_1$ and an Auditor with strategies in $S_2$. The payoff function of the game is $W: S_1 \times S_2 \rightarrow \R_{\ge 0}$.

\begin{definition} [Approximate Equilibrium \cite{Freund}]
\label{def:nuapprox}
A pair of strategies \\ $(s_1, s_2) \in S_1 \times S_2$ is said to be a $\nu$-approximate minimax equilibrium of the game if the following conditions hold: $U(s_1, s_2) - \min_{s'_1 \in S_1}  U(s'_1, s_2)  \le \nu, \quad
\max_{s'_2 \in S_2}  U(s_1, s'_2) -  U(s_1, s_2)  \le \nu$
\end{definition}

Freund and Schapire \cite{Freund} show that if a sequence of actions for the players jointly has low \textit{regret}, the uniform distribution over each player's actions forms an approximate equilibrium:

\begin{thm}[No-Regret Dynamics \cite{Freund}]\label{thm:noregret}
    Let $S_1$ and $S_2$ be convex, and suppose  $W(\cdot, s_2): S_1 \to \R_{\ge 0}$ is convex for all $s_2 \in S_2$ and $W(s_1, \cdot): S_2 \to \R_{\ge 0}$ is concave for all $s_1 \in S_1$. Let $(s_1^1, s_1^2, \ldots, s_1^T)$ and $(s_2^1, s_2^2, \ldots, s_2^T)$ be  sequences of actions for each player. If for $\nu_1,\nu_2 \ge 0$, the regret of the players jointly satisfies
    \[
    \sum_{t=1}^T W(s_1^t, s_2^t) - \min_{s_1 \in S_1} \sum_{t=1}^T W(s_1, s_2^t) \le \nu_1 T \quad
    \max_{s_2 \in S_2} \sum_{t=1}^T W(s_1^t, s_2) - \sum_{t=1}^T W(s_1^t, s_2^t) \le \nu_2 T
    \]
then  the pair $(\bar{s}_1, \bar{s}_2)$ is a $(\nu_1+\nu_2)$-approximate equilibrium, where      $\bar{s}_1 = \frac{1}{T}\sum_{t=1}^T s_1^t \in S_1$ and $\bar{s}_2 = \frac{1}{T}\sum_{t=1}^T s_2^t \in S_2$ are the uniform distributions over the action sequences.
\end{thm}

Additionally, we define a Cost Sensitive Classification (CSC) oracle over a classification model class $\Hs$, which we will use as an efficient subroutine in our algorithm.
\begin{definition}[Weighted Cost-Sensitive Classification Oracle for $\Hs$ \cite{DBLP:journals/corr/abs-1905-12843}]
An instance of a Weighted Cost-Sensitive Classification problem, or a $CSC$ problem, for the class $\Hs$, is given by a set of $n$ tuples $\{w(x_i), x_i,c_i^0,c_i^1 \}_{i=1}^{n}$ such that $c_i^1$ corresponds to the cost for predicting label 1 on sample $x_i$ and $c_i^0$ corresponds to the cost for prediction label $0$ on sample $x_i$. The weight of $x_i$ is denoted by $w(x_i)$. Given such an instance as input, a $CSC (\Hs)$ oracle finds a hypothesis $h \in \Hs$ that minimizes the total cost across all points:
$
h \in \argmin_{h' \in \Hs} \sum_{i=1}^{n} w(x_i) \left[h'(x_i)c_i^1 + (1-h'(x_i))c_i^0\right]
$.
\end{definition}

\section{Computing Sampling Weights from a Proxy (QP Approach)} \label{sec:approach}
Now we introduce our first methodological contribution: a selection approach for producing a balanced set \textit{given} a proxy. At a high level, our approach involves mapping each example to an acceptance probability. We construct such a mapping by labeling the range of the proxy $g:\mathcal{X}\rightarrow \mathbb{N}$ with acceptance probabilities and then selecting samples for our set by applying the proxy function to a sample and keeping it with probability corresponding to the element of the range of the proxy that the point maps to.

Recall the condition distribution matrix 
$A$, where each row represents the distribution of $z$ values mapped to a given proxy value. Our goal is to find acceptance probabilities such that the induced distribution on retained points is uniform over the protected attributes. By Lemma~\ref{lem:convex_hull}, such probabilities exist if $A$ contains the uniform distribution in its convex hull.  Consider the system $qA= U$, where $U = (\frac{1}{K},...,\frac{1}{K})$ and $q$ must be a length $\ell$ stochastic vector. If there is a solution for $q$, we consider this a valid acceptance rate scheme and use it to derive the selection probabilities for our filtering problem. If there is not an exact solution (which will happen frequently) we take $\argmin_q \lVert qA - U \rVert_2$ as our best acceptance rate scheme. \textit{Because this involves solving a quadratic program, we refer to the proxies and accompanying selection schemes produced by this approach as QP (Quadratic Program) proxies.}
\begin{algorithm}
\begin{algorithmic}
\REQUIRE proxy $g$, $D=\{(x_i,z_i)\}_{i=1}^{n}$, number of sensitive groups $K$
\FOR{$j$ in Range($g$)}
\STATE For $k$ in $[K]$, let $a_k = \Pr_{z|x \sim D}[z=k|g(x)=j]$\; 
\STATE Let $\hat{r}_j = \Pr_{x\sim D}[g(x) = j]$\; 
\ENDFOR
\STATE Let $A$ be the matrix with $k^{th}$ row $a_k$ and let $U=(\frac{1}{K},...,\frac{1}{K})$
\STATE $q = \argmin_q \lVert qA -U \rVert_2$ s.t. $q_i \geq 0$ and $\sum q_i = 1$
\STATE For $j$ in Range($g$), set $\rho_j = \frac{q_j}{\hat{r}_j} $ 
\STATE Let $C = \max_j{\rho_j}$ and normalize $\rho_j = \rho_j / C$
\RETURN $\rho$, $A$
\end{algorithmic}
\caption{Finding Acceptance Probabilities $\rho$}
\label{alg:deriving_probs}
\end{algorithm}
\begin{algorithm}
\begin{algorithmic}
\REQUIRE $g$, $\rho$, $\Ps_{x}$
\STATE Draw $x\sim \Ps_{x}$ and compute $g(x)$
\STATE With probability $\rho_{g(x)}$, accept $x$ into sample
\end{algorithmic}
\vspace{-0.1cm}
\caption{Filtering with $\rho$}
\label{alg:sampling}
\end{algorithm}
\begin{restatable}[Filtering According to $\rho$]{lemma}{sampling}
\label{ref:lemma_sampling}
Consider acceptance probabilities $\rho$ and conditional distribution matrix $A$ returned by Algorithm~\ref{alg:deriving_probs}. Then, if $U \in C(A)$, filtering according to $\rho$ 
 as in Algorithm~\ref{alg:sampling} induces a uniform distribution over protected attributes.
\end{restatable}

\begin{proof}

We want to show that the distribution over sensitive attributes in the \textit{filtered set} is uniform. We begin by expressing the distribution over sensitive attributes in the filtered set constructively, as the distribution obtained from sampling according to $\rho$. From there, we plug in our definitions of $a_{k,j}$ as the $j^{th}$ element in the $k^{th}$ row of the conditional distribution matrix of $z$ values given proxy values and as well as our definition of $\hat{r}_j$ as the marginal probability that a proxy value is $j$. Finally, we use the result that $qA=U=(\frac{1}{K}...\frac{1}{K})$
\begin{align*}
\sum_{j \in Range(g)} \rho_j \Pr[z=k,g(x)=j]
&= \sum_{j \in Range(g)} \rho_j \Pr[z=k|g(x)=j] \Pr[g(x)=j] \\&= \sum_{j \in Range(g)} a_{k,j} \hat{r}_j \rho_j = \sum_{j \in Range(g)} a_{k,j} q_j = \frac{1}{K}
\end{align*}
\end{proof}
\section{Learning an $(\alpha,\beta)$ Proxy}
\label{sec:proxies}
We have discussed a proxy function $g: \X \rightarrow \mathbb{N}$ that maps samples to proxy groups and described the conditional distribution matrix $A$ indicating the distribution of sensitive attributes \textit{within} each proxy group. In Section~\ref{sec:approach}, we showed how $A$ can be used to derive acceptance probabilities for each group, such that under appropriate conditions, selecting according to these probabilities induces a uniform distribution over the protected attributes. Up until now, however, we have referenced $A$ as fixed --- we have used it to derive retention probabilities but have not described how it and the proxy can be generated. Recall that our proxy function $g \in \mathcal{G}$ takes the form of a decision tree, where each leaf is a \textit{proxy group}. Therefore, each row in $A$, corresponding to the distribution over sensitive attributes in a given \textit{proxy group}, also corresponds to the distribution over these attributes in a given \textit{leaf}.

We grow our decision tree by sequentially making \textit{splits} over the feature space --- our tree will start as a stump and our matrix will have just one row, then we will split the tree into two leaves and the matrix will have two rows, and we will continue in this manner, splitting a leaf (and adding a row to the matrix) at each iteration. We will make these splits by employing a classification function from the pre-specified model class $\Hs \subseteq \{h: X \rightarrow \{0,1\}\}$ assigned to each leaf. Because the two representations, as a matrix or a tree, afford different analytical advantages, we will continue to refer to both as we derive our algorithm. One advantage of the matrix representation is that it allows us to reason about the convex hull of a set of conditional distributions. Lemma~\ref{lem:convex_hull} showed that there is a solution to $qA=U$ for a stochastic vector $q$ if and only if $U$ lies in the convex hull of $A$. Our goal will be to grow our tree (and the matrix $A$) so that the $\inf_{U' \in C(A)} \lVert U' - U \rVert_2$ shrinks at each iteration --- until finally $U$ is contained within (or sufficiently close to) $C(A)$.\footnote{Algorithms~\ref{alg:deriving_probs} and \ref{alg:sampling} and Lemma~\ref{ref:lemma_sampling} extend easily to distributions other than the uniform.}

We begin with a geometric interpretation of $C(A)$ and describe how it changes as our tree and conditional distribution matrix expand. In particular, we grow a tree that has leaves $V$ and keep track of the corresponding matrix $A$ of sensitive attribute distributions conditional on their classification by the tree. We can always label the leaves of a tree with a binary sequence, so from now on we will identify each $V$ with a binary sequence. Using this description, we derive sufficient conditions to decrease the Euclidean distance between $C(A)$ and $U$. We begin with several definitions that we will use to characterize $C(A)$.

\begin{definition}[Vertex]
\label{def:vertex}
Let $R$ be a bijective mapping of vertices to a rows in $A$. Then $V \in \{0,1\}^\mathbb{N}$ is a vertex of $C(A)$ if $R(V)$ corresponds to a row $a_i$ such that $a_i \notin C(\{a_j\}_{j < i})$.
\end{definition}
Note that in our context this means that each \textit{row} of $A$ corresponds to a \textit{vertex} of $C(A)$ as long as it cannot be represented as a convex combination of the other rows. Next, we introduce the function that is used at a node of the decision tree to partition samples into the left or right child. It will also be convenient in our algorithm to make use of randomized splitting functions, so we handle both cases.
\begin{definition}[Splitting Function]
\label{def:split_fn}
We call $h_V \in \Hs$ a deterministic \textit{splitting function} at vertex $V$. A \textit{randomized splitting function} $\tilde{h}_V \in \Delta \Hs$ is a distribution supported on a finite set of deterministic splitting functions $\{h_V^i\}_{i=1}^{n}$ such that $\tilde{h}_V(x) = h_V^i(x)$ with probability $\frac{1}{n}$ for all $i$.
\end{definition}
Each vertex $V$ is paired with a splitting function $\tilde{h}_V$ operating on samples mapped to $V$. To model the \textit{expected} action of a randomized splitting function, we introduce the notion of \textit{sample weights}, where the weight of a sample $x$ at $V$ is the probability that $x$ reaches $V$ in its random walk down the tree (as determined by the randomized splitting function). Here, $V\backslash 0$ indicates the parent of $V$ if $V$ ends in 0, and $V\backslash 1$ indicates the parent if $V$ ends in 1. Note that because $V$ is a binary sequence, we can apply the modulo operator with the binary representation of 2 to isolate the last digit.
\begin{definition}[Sample Weights]
The \textit{weight} of a sample $x$ at vertex $V$ is defined as follows:
$w_0(x)=1$ and for $V \neq 0$,
  $w_V(x) = \left\{\begin{array}{lr}
        w_{V\backslash 0}(x) \cdot \mathbb{E}[\tilde{h}_{V\backslash 0}(x)] & \text{if }V\mod 2 =0\\
        w_{V\backslash 1}(x) \cdot \mathbb{E}[1-\tilde{h}_{V\backslash 1}(x)] & \text{if }V\mod 2 = 1 \\
        \end{array}\right\}$
\end{definition}
We distinguish between $V$ and the collection of weighted samples represented by V, $l_V$. 
\begin{definition}[Collection of Weighted Samples at $V$]
Given randomized splitting functions $\{\tilde{h}_i\}_{i=0}^{V}$, the collection of weighted samples at $V$ is denoted by $\scalemath{0.95}{l_V = \{w_V(x),(x,z):(x,z) \in S\}}$.
\end{definition}
\begin{definition}[Vertex Split]
A vertex split results from applying $\tilde{h}_V$ to $x \in l_V$, where 
$l_{V0} = \{w_{V0}(x),(x,z):(x,z) \in S\}$ and $l_{V1} = \{w_{V1}(x), (x,z):(x,z) \in S\}$.
\end{definition}
After V is split into V0 and V1, V is no longer a vertex, whereas V0 and V1 may be. So, the number of leaves in the tree, and therefore the number of rows in $A$, increased by at most 1.
\subsection{Growing the Convex Hull and Learning a Splitting Function}
\label{sec:desiderata}
Imagine that we have started to grow our proxy tree, but $U$ is not in $C(A)$. We would like to expand $C(A)$ to contain $U$, and intuitively, we might like to expand $C(A)$ in the direction of $U$. One way to do so is to choose a vertex $V$ to split into two vertices, $V1$ and $V0$. We assume that $V1$ is the split such that $R(V1) - R(V)$ is most in the direction of $U - U'$, where $U'$ is the closest point in Euclidean distance to $U$ in $C(A)$.
\begin{definition}[Convex Hull Notation]
Let $\theta$ be the angle between $R(V1) - R(V)$ and $U-U'$, and $U''$ be the closest point to $U$ on the line segment $R(V1) - U'$:

\begin{minipage}{0.5\textwidth}
\begin{align*}
&U' = \arg \min_{U^* \in C(A)} \lVert U - U^*\rVert_2 \\& \cos \theta = \frac{\langle R(V1) - R(V), U - U'\rangle}{\lVert R(V1) - R(V)\rVert_2  \lVert U - U'\rVert_2 } \\ &U'' = tU' + (1-t)R(V1) \text{ where}\\& t = \argmin_{0<t^*<1} \lVert U - (t^* U' + (1-t^*)R(V1))\rVert_2
\end{align*}
\end{minipage}
\begin{minipage}{0.45\textwidth}
\resizebox{\textwidth}{!}{
\begin{tikzpicture}[x=0.75pt,y=0.75pt,yscale=-1,xscale=1, scale=0.5]

\draw   (349.41,6.07) -- (693.67,490.1) -- (1.94,493.24) -- cycle ;
\draw  [color={rgb, 255:red, 208; green, 2; blue, 27 }  ,draw opacity=1 ] (247,386.5) -- (693.67,490.1) -- (1.94,490.1) -- cycle ;
\draw   (343.23,269.89) -- (343.23,283.53) -- (359.53,279.32) -- (343.23,283.53) -- (353.3,294.57) -- (343.23,283.53) -- (333.16,294.57) -- (343.23,283.53) -- (326.93,279.32) -- (343.23,283.53) -- cycle ;
\draw  [dash pattern={on 4.5pt off 4.5pt}] (344.34,277.3) -- (312.3,401.34) -- (244.53,383.83) ;
\draw   (296.98,376.58) -- (317.02,381.29) -- (312.3,401.34) -- (292.26,396.62) -- cycle ;
\draw [color={rgb, 255:red, 74; green, 144; blue, 226 }  ,draw opacity=1 ] [dash pattern={on 4.5pt off 4.5pt}]  (238.5,160.5) -- (247,386.5) ;
\draw [color={rgb, 255:red, 74; green, 144; blue, 226 }  ,draw opacity=1 ] [dash pattern={on 4.5pt off 4.5pt}]  (247,386.5) -- (251.89,496.51) ;
\draw [color={rgb, 255:red, 144; green, 19; blue, 254 }  ,draw opacity=1 ] [dash pattern={on 4.5pt off 4.5pt}]  (238.5,154) -- (312.3,401.34) ;
\draw [color={rgb, 255:red, 144; green, 19; blue, 254 }  ,draw opacity=1 ]   (290,301.5) -- (343.23,283.53) ;
\draw  [color={rgb, 255:red, 144; green, 19; blue, 254 }  ,draw opacity=1 ] (298.18,298.08) -- (303.25,313.22) -- (288.11,318.29) -- (283.04,303.15) -- cycle ;
\draw  [color={rgb, 255:red, 208; green, 2; blue, 27 }  ,draw opacity=1 ][fill={rgb, 255:red, 208; green, 2; blue, 27 }  ,fill opacity=1 ] (680.67,490.1) .. controls (680.67,486.51) and (683.58,483.6) .. (687.17,483.6) .. controls (690.76,483.6) and (693.67,486.51) .. (693.67,490.1) .. controls (693.67,493.69) and (690.76,496.6) .. (687.17,496.6) .. controls (683.58,496.6) and (680.67,493.69) .. (680.67,490.1) -- cycle ;
\draw  [color={rgb, 255:red, 208; green, 2; blue, 27 }  ,draw opacity=1 ][fill={rgb, 255:red, 208; green, 2; blue, 27 }  ,fill opacity=1 ] (240.5,386.5) .. controls (240.5,382.91) and (243.41,380) .. (247,380) .. controls (250.59,380) and (253.5,382.91) .. (253.5,386.5) .. controls (253.5,390.09) and (250.59,393) .. (247,393) .. controls (243.41,393) and (240.5,390.09) .. (240.5,386.5) -- cycle ;
\draw  [color={rgb, 255:red, 208; green, 2; blue, 27 }  ,draw opacity=1 ][fill={rgb, 255:red, 208; green, 2; blue, 27 }  ,fill opacity=1 ] (1.94,493.24) .. controls (1.94,489.65) and (4.85,486.74) .. (8.44,486.74) .. controls (12.03,486.74) and (14.94,489.65) .. (14.94,493.24) .. controls (14.94,496.83) and (12.03,499.74) .. (8.44,499.74) .. controls (4.85,499.74) and (1.94,496.83) .. (1.94,493.24) -- cycle ;
\draw  [color={rgb, 255:red, 74; green, 144; blue, 226 }  ,draw opacity=1 ][fill={rgb, 255:red, 74; green, 144; blue, 226 }  ,fill opacity=1 ] (232,160.5) .. controls (232,156.91) and (234.91,154) .. (238.5,154) .. controls (242.09,154) and (245,156.91) .. (245,160.5) .. controls (245,164.09) and (242.09,167) .. (238.5,167) .. controls (234.91,167) and (232,164.09) .. (232,160.5) -- cycle ;
\draw  [color={rgb, 255:red, 74; green, 144; blue, 226 }  ,draw opacity=1 ][fill={rgb, 255:red, 74; green, 144; blue, 226 }  ,fill opacity=1 ] (245.39,490.01) .. controls (245.39,486.42) and (248.3,483.51) .. (251.89,483.51) .. controls (255.48,483.51) and (258.39,486.42) .. (258.39,490.01) .. controls (258.39,493.6) and (255.48,496.51) .. (251.89,496.51) .. controls (248.3,496.51) and (245.39,493.6) .. (245.39,490.01) -- cycle ;
\draw  [color={rgb, 255:red, 0; green, 0; blue, 0 }  ,draw opacity=1 ][fill={rgb, 255:red, 0; green, 0; blue, 0 }  ,fill opacity=1 ] (305.8,401.34) .. controls (305.8,397.75) and (308.71,394.84) .. (312.3,394.84) .. controls (315.89,394.84) and (318.8,397.75) .. (318.8,401.34) .. controls (318.8,404.93) and (315.89,407.84) .. (312.3,407.84) .. controls (308.71,407.84) and (305.8,404.93) .. (305.8,401.34) -- cycle ;
\draw  [dash pattern={on 4.5pt off 4.5pt}] (274.57,271.05) -- (244.49,387.53) -- (240.5,386.5) ;
\draw  [color={rgb, 255:red, 144; green, 19; blue, 254 }  ,draw opacity=1 ][fill={rgb, 255:red, 144; green, 19; blue, 254 }  ,fill opacity=1 ] (276.54,303.15) .. controls (276.54,299.56) and (279.45,296.65) .. (283.04,296.65) .. controls (286.63,296.65) and (289.54,299.56) .. (289.54,303.15) .. controls (289.54,306.74) and (286.63,309.65) .. (283.04,309.65) .. controls (279.45,309.65) and (276.54,306.74) .. (276.54,303.15) -- cycle ;

\draw (1,514) node [anchor=north west][inner sep=0.75pt]   [align=left]
{};
\draw (514,514) node [anchor=north west][inner sep=0.75pt]   [align=left]
{}; 
\draw (160,120) node [anchor=north west][inner sep=0.75pt]   [align=left] {$R(V1)$};
\draw (361.53,282.32) node [anchor=north west][inner sep=0.75pt]   [align=left] {$U$};
\draw (170,360) node [anchor=north west][inner sep=0.75pt]   [align=left] 
{$R(V)$};
\draw (325,375) node [anchor=north west][inner sep=0.75pt]   [align=left] {$U'$};
\draw (259,450) node [anchor=north west][inner sep=0.75pt]   [align=left] {$R(V0)$};
\draw (240,325) node [anchor=north west][inner sep=0.75pt]    {$\theta $};
\draw (310,295.92) node [anchor=north west][inner sep=0.75pt]    {$\phi $};
\draw (284,250) node [anchor=north west][inner sep=0.75pt]   [align=left] {$U''$};
\end{tikzpicture}
}
\label{fig:temp}
\end{minipage}
\end{definition}
We show that, given certain assumptions, we can lower bound how much this splitting process will decrease the distance from $C(A)$ to $U$. The first condition in Lemma~\ref{lem:vertex_splitting} will be used to derive an objective function over which we can optimize to find a splitting function. The second and third conditions limit the theory to the case where we can prove our progress lemma. The second condition says that the distance between $R(V)$ and $R(V1)$ has to be sufficiently large compared to the existing distance between the uniform distribution and its projection onto $C(A)$. The third condition is needed for the proof, allowing us to make arguments based on right triangles --- it is satisfied when the second condition is met and the angle between $R(V1) - R(V)$ and $U - U'$ is not too large. As these conditions are potentially limiting theoretically, we verify that they are indeed frequently satisfied in the experiments.
\begin{restatable}[Progress via Vertex Split]{lemma}{vertexSplitting}
\label{lem:vertex_splitting}
When a vertex $V$ is split, forming new vertices $V0$ and $V1$, the distance from the convex hull to $U$ decreases by at least a factor of $1 - \gamma$ if
\begin{align}
\label{eq:sufficient_condition}
&\langle R(V1) - R(V), U - U'\rangle /\lVert U - U'\rVert_2 \geq f(\gamma)\text{ and}
\\&\lVert R(V1) - R(V) \rVert_2 \geq (1-\gamma)^{-1} \sqrt{2\gamma - \gamma^2} \lVert U - U'\rVert_2,\quad R(V1) - U' \perp U - U''\; \text{ where} \nonumber
\end{align}
\[\scalemath{0.95}{f(\gamma):= \sqrt{\left(2\gamma - \gamma^2\right)\left(2 - \lVert R(V)- U'\rVert_2^2 + 2\lVert R(V)- U'\rVert_2 (1-\gamma)\sqrt{(\gamma^2 - 2\gamma)\lVert R(V)- U'\rVert_2^2 + 2}\right)}}\]
\end{restatable}
\begin{remark}
These are sufficient, but not necessary, conditions for a split to make sufficient progress. Empirically, we simply require that each split decreases the distance from the convex hull to $U$ by at least a factor of $1-\gamma$ for the algorithm to continue. 
\end{remark}
To summarize, these conditions ask that we split a vertex of the convex hull (equivalently a leaf of the proxy tree), so that the convex hull expands in the direction of the target vector. In other words, we want to split a leaf into the over-represented groups in one child and the under-represented groups in the other child, without violating the disclosure constraints. Lemma~\ref{eq:sufficient_condition} also involves conditions that make sure that this split is sufficiently large to move the convex hull closer to the uniform rather than making minute progress.
Having identified a sufficient condition for a split to make suitable progress toward containing the uniform distribution within the convex hull, we present a subroutine to find an $\alpha$-proxy. We first express Equation~\eqref{eq:sufficient_condition} in a form amenable to use in a linear program:
\begin{restatable}[Objective Function]{lemma}{objectiveFunction}
\label{lemma:objectiveFunction}
Let $m_V$ be the number of samples in $l_V$ and let $h_V$ be the splitting function for vertex $V$. The condition $\frac{\langle R(V1) - R(V), U - U'\rangle}{\lVert U - U'\rVert_2}  \geq f(\gamma)$ is equivalent to 
\begin{align*}
&\sum_{i=1}^{m_V} w_V(x_i)h_V(x_i) ( - Q + \sum_{k=1}^{K} \1_{z_i = k} \left(U_k' - U_k\right)) \leq~0 \\& \text{for } Q_{V,U',\gamma}:=\lVert U' - U \rVert f(\gamma) + \frac{\sum_{i=1}^{m_V} w_V(x_i)\sum_{k=1}^{K} \1_{z_i = k} \left(U_k' - U_k\right)}{\sum_{i=1}^{m_V} w_V(x_i)}
\end{align*}
\end{restatable}

\begin{proof}
We begin by expanding the scaled dot product between $R(V1) - R(V)$ and $U-U'$:
\begin{equation*}
\scalemath{0.9}{\frac{\langle R(V1) - R(V), U - U'\rangle}{\lVert U - U'\rVert } = \sum_{j=1}^{m_V} w_V(x_j) \sum_{k=1}^{K} \1_{z_j = k} \left(\frac{h_V(x_j)}{\sum_{j=1}^{m_V} w_V(x_j) h_V(x_j)} - \frac{1}{\sum_{j=1}^{m_V} w_V(x_j) }\right) \frac{U_k - U_k'}{\lVert U - U'\rVert}}
\end{equation*}

Asking $\frac{\langle R(V1) - R(V), U - U'\rangle}{\lVert U - U'\rVert}  \geq f(\gamma)$ is equivalent to asking $\frac{\langle R(V1) - R(V), U' - U\rangle}{\lVert U' - U\rVert }\leq~f(\gamma)$ or:

\begin{equation}
\label{eq:objective_constraint}
\scalemath{0.91}{\frac{\sum_{j=1}^{m_V} w_V(x_j)h_V(x_j) \sum_{k=1}^{K} \1_{z_j = k} \left(U_k' - U_k\right)}{\sum_{i=1}^{m_V} w_V(x_j) h_V(x_j)} \leq  \lVert U'- U \rVert f(\gamma) +\frac{\sum_{j=1}^{m_V} w_V(x_j)\sum_{k=1}^{K} \1_{z_j = k} \left(U_k' - U_k\right)}{\sum_{j=1}^{m_V} w_V(x_j)}}
 \end{equation}

Finally, the right-hand side is constant given $V$, $U'$, and $\gamma$. Therefore, we represent it by a constant $Q_{V,U',\gamma}:=\lVert U' - U \rVert f(\gamma) + \frac{\sum_{j=1}^{m_V} w_V(x_j)\sum_{k=1}^{K} \1_{z_j = k} \left(U_k' - U_k\right)}{\sum_{j=1}^{m_V} w_V(x_j)}$. This allows us to rewrite Equation~\eqref{eq:objective_constraint} as 
\begin{align*}
\sum_{j=1}^{m_V} w_V(x_j)h_V(x_j) \left( - Q_{V,U',\gamma} + \sum_{k=1}^{K} \1_{z_j = k} \left(U_k' - U_k\right)\right) \leq~0
\end{align*}
\end{proof}

We use Lemma~\ref{lemma:objectiveFunction} to form a cost-sensitive classification problem for vertex $V$, where the constraints make sure that any candidate proxy is no more than $\alpha$-disclosive:

\begin{align}
\label{eqn:program}
&\min_{h_V\in \Hs} \sum_{i=1}^{m_V} w_V(x_i) h_V(x_i) \left( - Q_{V,U',\gamma} + \sum_{k=1}^{K} \1_{z_i = k} \left(U_k' - U_k\right) \right) \text{ s.t. }\forall k\\
&\lvert \frac{\sum_{i=1}^{m_V} w_V(x_i)h_V(x_i) \1_{z_i = k}}{\sum_{i=1}^{m_V} w_V(x_i) h_V(x_i)} - r_k\rvert \leq \alpha \text{ and }\lvert\frac{\sum_{i=1}^{m_V} w_V(x_i)(1-h_V(x_i)) \1_{z_i = k}}{\sum_{i=1}^{m_V} w_V(x_i)(1-h_V(x_i))}- r_k\rvert \leq \alpha \nonumber
\end{align}

Next, we will appeal to strong duality to derive the corresponding Lagrangian. We note that computing an approximately optimal solution to the linear program corresponds to finding approximate equilibrium strategies for both players in the game in which one player, the ``Learner,'' controls the primal variables and aims to minimize the Lagrangian value. The other player, the ``Auditor,'' controls the dual variables and seeks to maximize the Lagrangian value. If we construct our algorithm in such a way that it simulates repeated play of the Lagrangian game such that both players have sufficiently small regret, we can apply  Theorem~\ref{thm:noregret} to conclude that our empirical play converges to an approximate equilibrium of the game. Furthermore, our algorithm will be \textit{oracle efficient}: it will make polynomially many calls to oracles that solve weighted cost-sensitive classification problems over $\Hs$. 

To turn Program~\eqref{eqn:program} into a form amenable to our two-player zero-sum game formulation, we expand $\Hs$ to $\Delta \Hs$, allow our splitting function to be \textit{randomized}, and take expectations over the objective and constraints with respect to deterministic splitting functions drawn according to $\tilde{h}_V$. Doing so yields the following CSC problem to be solved for vertex $V$:

\begin{mini}|s|
{\tilde{h}_V\in \Delta \Hs}{\E_{h_V \sim \tilde{h}_V}\sum_{i=1}^{m_V} w_V(x_i)h_V(x_i)\left( - Q_{V,U',\gamma} + \sum_{k=1}^{K} \1_{z_i = k} \left(U_k' - U_k\right))\right)}{}{}
\addConstraint{\E_{h_V \sim \tilde{h}_V}\sum_{i=1}^{m_V} w_V(x_i) \tilde{h}_V(x_i) \left( \1_{z_i = k} - r_k - \alpha \right)}{\leq 0 \; \forall k}{}
\addConstraint{\E_{h_V \sim \tilde{h}_V}\sum_{i=1}^{m_V} w_V(x_i) (1-h_V(x_i)) \left( \1_{z_i = k} - r_k - \alpha \right)}{\leq 0 \; \forall k}{}
\addConstraint{\E_{h_V \sim \tilde{h}_V}\sum_{i=1}^{m_V} w_V(x_i) h_V(x_i) \left( r_k - \1_{z_i = k} - \alpha \right)}{\leq 0 \; \forall k}{}
\addConstraint{\E_{h_V \sim \tilde{h}_V}\sum_{i=1}^{m_V} w_V(x_i)(1-h_V(x_i)) \left( r_k - \1_{z_i = k} - \alpha \right)}{\leq 0 \; \forall k}{}
\label{eqn:expected_program}
\end{mini}

We solve this constrained optimization problem by simulation a zero-sum two-player game on the Lagrangian dual. Given dual variables $\lambda \in \R_{\geq 0}^{4K}$ such that $\lVert \lambda \rVert_2 \leq \lambda_{max}$ for some constant $\lambda_{max}$, the Lagrangian of Program~\eqref{eqn:expected_program} is:

\begin{align*}
L(\lambda,\tilde{h}_V) = \E_{h_V \sim \tilde{h}_V} \sum_{i=1}^{m_V} w_V(x_i) &\left(- Q_{V,U',\gamma} h_V(x_i) + \sum_{k=1}^{K} h_V(x_i) \1_{z_i = k} \left(U_k' - U_k\right) + \right. \nonumber \\ 
& \left. \left(\lambda_{k,1} h_V(x_i) + \lambda_{k,0} (1-h_V(x_i))\right) \left( \1_{z_i = k} - r_k - \alpha \right) + \right. \\& \left. \left(\lambda_{k,3} h_V(x_i) + \lambda_{k,2} (1-h_V(x_i))\right) \left( r_k - \1_{z_i = k} -  \alpha \right)\nonumber\right)
\end{align*}

Given the Lagrangian, solving Program~\eqref{eqn:expected_program} is equivalent to solving the minimax problem
    $\min_{\tilde{h}_V \in \Delta \Hs} \max_{\lambda \in \R_{\geq 0}^{4K} }L(\lambda,\tilde{h}_V) =  \max_{\lambda \in \R_{\geq 0}^{4K}} \min_{\tilde{h}_V \in \Delta \Hs} L(\lambda,\tilde{h}_V)$, where the minimax theorem holds because the range of the primal variable, i.e., $\Delta \Hs$ is convex and compact, the range of the dual variable, i.e., $\R_{\geq 0}^{4K}$ is convex, and the Lagrangian function $L$ is linear in both primal and dual variables. Therefore, we focus on solving the minimax problem, which can be seen as a two-player zero-sum game between the primal player (the Learner) who is controlling $\tilde{h}_V$ and the dual player (the Auditor) who is controlling $\lambda$. Using no-regret dynamics, we will have the Learner deploy its best response strategy in every round, which will be reduced to a call to $CSC(\Hs)$ and let the Auditor with strategies in $\Lambda=\{ \lambda: 0 \leq \lambda \leq \lambda_{max} \}$ play according to Online Projected Gradient Descent \cite{zinkevich}.

Our local algorithm for splitting a vertex is described in Algorithm~\ref{alg:split}, and its guarantee is given in Theorem~\ref{thm:split}. We note that the algorithm returns a distribution over $\Hs$. Given an action $\lambda$ of the Auditor, we write $LC(\lambda)$ for the vector of costs for labeling each data point as 1. We view our costs as the inner product of the outputs of a deterministic splitting function $h_V$ on the $m_V$ points and corresponding cost vector. We define the cost for labeling an example $0$ to be $0$ for all $x$ ($c^0(x)=0$), and the cost for labeling an example 1 as:
\begin{align*}
c^1(x) = w_V(x) ( -Q_{V,U',\gamma} + \sum_{k=1}^{K} \1_{z_j = k} \left(U_k' - U_k \right) & \left.+ \left( \lambda_{k,1} -\lambda_{k,0} \right) \left(\1_{z = k} - r_k - \alpha \right)  \right. \\& \left. +  \left( \lambda_{k,3} -\lambda_{k,2} \right) \left(r_k - \1_{z = k} - \alpha \right) \right)
\end{align*}

\begin{algorithm}
\begin{algorithmic}
\REQUIRE $\{w_V(x_i),(x_i, z_i)\}_{i=1}^{m_V}$, model class $\Hs$, $CSC(\Hs)$, $\alpha$, $\epsilon$, $\gamma$
\STATE Set $\lambda_{max}= m(K-1)/K\epsilon +2 $ and $T = \lceil \left( 2 K m \left(1+\alpha\right) \lambda_{\max}/\epsilon\right) ^2 \rceil$
\STATE Initialize $\lambda_k$ = 0 $\forall k$
\FOR{$t=1 \dots T$}
\STATE $h_t = \argmin_{h \in \Hs} \langle LC(\lambda), h \rangle$
\STATE $\lambda_t = \lambda_{t-1} + t^{-1/2}\left( \nabla_{\lambda} L \right)^+$; If $\lVert \lambda_t \rVert > \lambda_{\max}$, set $\lambda_t = \lambda_{\max} \frac{\lambda_t}{\lVert \lambda_t \rVert}$
\ENDFOR
\RETURN $\tilde{h}_V := \text{uniform distribution over }h_V^t$
\end{algorithmic}
\caption{Learning a Splitting Function}
\label{alg:split}
\end{algorithm}
\begin{restatable}[Learning an $(\alpha + \epsilon)$-Disclosive Proxy]{thm}{split}
\label{thm:split}
Fix $\alpha$, $\epsilon$, suppose $\Hs$ has finite $VC$ dimension, and suppose $\exists \;\tilde{h}_V^* \in \Delta \Hs$ that is a feasible solution to Program~\eqref{eqn:expected_program}. Then, Algorithm~\ref{alg:split} returns a distribution $\tilde{h}_V$ that is an $\epsilon$-optimal solution to Program~\eqref{eqn:expected_program}.
\end{restatable}
Theorem~\ref{thm:split} says that with appropriate conditions on the model class $\Hs$ and access to $CSC(\Hs)$, Algorithm~\ref{alg:split} returns a model satisfying the conditions of Program~\ref{eqn:expected_program} (i.e. produces an acceptable split) up to an additive factor of $\epsilon$. A few requirements of this theorem may not hold in practice and thus motivate our experiments. The choice of a base model class $\mathcal{H}$ impacts whether a feasible solution exists --- typically more complex model classes will be more likely to contain a feasible solution, but this complexity will impact the generalization bounds. Also, the guarantee relies on Algorithm~\ref{alg:split} having access to a cost sensitive classification oracle. In practice, we typically do not have such an oracle so must use a heuristic.
\subsection{Decision Tree Meta-Algorithm}\label{sec:tree}
Finally, we use these results to greedily construct a proxy $g: \mathcal{X} \rightarrow \mathbb{N}$. We do this iteratively using a decision tree, where leaves correspond to proxy groups. We split the data into these leaves in such a way that when we consider the distribution of groups in each leaf, the uniform vector is contained in their convex hull. This allows us to select a balanced set in expectation. In addition, we require that the proxy be $\alpha$-disclosive at every step. We grow the tree as follows, for some tolerance $\beta$: (1) If $\lVert U - C(A)\rVert_2 \leq \beta$, output the tree. (2) Otherwise, look for a leaf to split. If we find a suitable split, make it, and continue. If not, output the tree. To determine if a split is suitable, we use the results from Section~\ref{sec:desiderata}: for fixed approximation factor $\epsilon$, disclosivity budget $\alpha - \epsilon$, and progress parameter $\gamma$, a splitting function $\tilde{h}_V$ must be an $\epsilon$-approximate solution to Program~\ref{eqn:expected_program} (and therefore no more than $\alpha$-disclosive) at vertex $V$. If we can find such an $\tilde{h}_V$ for at least $\frac{\ln{\beta} - \ln{\sqrt{2}}}{\ln{1- \gamma}}$ rounds, the decision tree will be an ($\alpha, \beta$) proxy.
\begin{algorithm}
\begin{algorithmic}
\REQUIRE $D = \{x_i,z_i\}_{i=1}^{n}$, $CSC({\mathcal H})$, $\alpha$, $\epsilon$, $\gamma$, $\beta$
\WHILE{$\inf_{U' \in C(A)}\lVert U' - U\rVert_2 > \beta$}
\STATE Apply Algorithm~\ref{alg:split} to find feasible split (if no feasible split, terminate)
\STATE Expand tree $T$ and re-calculate $A$, $C(A)$
\ENDWHILE 
\RETURN $T$, $A$
\end{algorithmic}
\caption{Learning an ($\alpha, \beta$) Proxy}
\label{alg:decision_tree}
\end{algorithm}

\begin{restatable}[Learning an ($\alpha, \beta$) Proxy]{thm}{numRounds}
\label{thm:rounds}
If the conditions of Lemma~\ref{lem:vertex_splitting} are satisfied at every split, Algorithm~\ref{alg:decision_tree} produces an $(\alpha,\beta)$ proxy in-sample within $\frac{\ln{\beta} - \ln{\sqrt{2}}}{\ln{1- \gamma}}$ rounds.
\end{restatable}

\begin{proof}
    Let $A_{i^*}$ be the conditional distribution matrix returned by Algorithm~\ref{alg:decision_tree} after $i^*$ rounds. Our goal is to produce $A_{i^*}$ such that $\lVert U - C(A_{i^*}) \rVert_2\leq \beta $. Let $A_0$ be the initial conditional distribution matrix, and observe that if we decrease the distance from the current conditional distribution matrix to $U$ by a factor of $1-\gamma$ each round, at round $i$, $\lVert U - C(A_i)\rVert_2 \leq (1-\gamma)^i \lVert U - C(A_0) \rVert_2$. Further, recall that $\lVert U - C(A_0) \rVert_2 \leq \sqrt{2}$ because both $U$ and $C(A_0)$ must lie in the unit simplex. Setting  $\lVert U - C(A_{i^*}) \rVert_2\leq \beta $, we have $\beta \leq (1-\gamma)^{i^*} \sqrt{2} \implies \frac{\beta}{\sqrt{2}} \leq (1-\gamma)^{i^*} \implies i* \geq \frac{\ln{\beta} - \ln{\sqrt{2}}}{\ln{(1- \gamma)}}$. Then, after $i^* = \frac{\ln{\beta} - \ln{\sqrt{2}}}{\ln{(1- \gamma)}}$ rounds, $\lVert U - C(A_{i^*}) \rVert_2 \leq \gamma$. Finally, because our linear program constrains splits to only those that guarantee $\alpha$-disclosiveness, the final proxy must be $\alpha$-disclosive in-sample.
\end{proof}

Theorem~\ref{thm:rounds} allows us to upper bound the number of times that Algorithm~\ref{alg:decision_tree} performs a split and, therefore, the number of unique proxy groups generated. The theorem's hypothesis states, informally, that it must be possible to find a splitting function at each round that makes both a sufficiently \textit{large} split (i.e. the new vertex is sufficiently far from the old vertex compared to the current distance from the convex hull to the target uniform) and the split is sufficiently in the direction of the target. This theorem, in turn, allows us to state generalization bounds depending on both the number and size of each proxy group.
\begin{restatable}[Generalization]{thm}{generalization}
\label{thm:generalization}
 Let $\epsilon, \delta, \gamma>0$ and $G$ be the proxy class. Let there be $K$ sensitive groups. If each proxy group has at least $\frac{1}{2\epsilon^2}\ln{\frac{8K\cdot VC(\mathcal{G}) \left(\ln{\beta} - \ln\sqrt{2}\right)}{\delta \ln{(1- \gamma)}}}$ samples, with probability $1-\delta$, an $(\alpha,\beta)$ proxy in-sample will be an $(\alpha + 2\epsilon, \beta + K\epsilon\sqrt{\frac{\ln{\beta} - \ln{\sqrt{2}}}{\ln{(1- \gamma)}}})$ proxy out-of-sample.
\end{restatable}
Theorem~\ref{thm:generalization} presents the number of samples needed in each proxy group to obtain a sufficiently small generalization gap in both the disclosure level and imbalance --- it is based on the size of the \textit{smallest proxy group in-sample}, which might get quite small in practice. Furthermore, the generalization gap for $\beta$ scales by an additive factor of the number of sensitive groups, $K$. Therefore, as our problem becomes more challenging, more samples are required to achieve a proxy that performs similarly out-of-sample compared to in-sample.
\section{Experiments}
\label{sec:experiments}
Here, we test our two main methodological contributions. The first is to use Algorithms~\ref{alg:deriving_probs} and~\ref{alg:sampling} to solve $\min_q \lVert qA -U \rVert_2 \text{ subject to } q_i \geq 0 \; \forall i \text{ and }\sum_i q_i = 1$ and derive the corresponding acceptance probabilities $\rho$ for the given proxy. The second is to additionally use Algorithms~\ref{alg:split} and~\ref{alg:decision_tree} to learn a decision-tree proxy \textit{guaranteed} not to exceed a specified level of disclosure. 
\begin{enumerate}
\item \textit{QP Regression and Decision Tree} Proxies: We train a multinomial logistic regression model or decision tree to \textit{directly predict} the sensitive attribute but select our acceptance probabilities by employing Algorithms \ref{alg:deriving_probs} and \ref{alg:sampling}.
\item $(\alpha,\beta)$ Proxy: We use Algorithm~\ref{alg:decision_tree} to develop a proxy for a specified disclosure budget.
\end{enumerate}
We compare the performances of these proxy functions against those of two baselines:\footnote{For the Naive Proxies, QP Proxies, and SMOTE, we interpolate between a uniform and proxy-specific sampling strategy by post-processing: We predict $z$ with the proxy and then, with probability $\eta \in [0,1]$, uniformly re-assign the prediction. Finally, we apply Algorithm~\ref{alg:sampling} to sample according to the post-processed proxy labels and plot the balance and disclosure of the corresponding data set \textit{with respect to the post-processed proxy values}. We use a large point marker for the results without post-processing} 
\begin{enumerate}
\item \textit{Naive Regression and Decision Tree} proxies: We train models to \textit{directly predict} sensitive attributes then sample the same number of points from each predicted group, inducing a conditional distribution matrix of the distribution of sensitive attributes in each proxy group. We then calculate the degree of disclosure and imbalance of the sampled set.
\item SMOTE~\cite{smote}: We train a decision tree to \textit{directly predict} groups and then, using these predictions as input for SMOTE, balance the data by synthesizing minority examples.
\end{enumerate}

\subsection{Data, Hyperparameters, and Compute Time}
We evaluate the disclosure, $\alpha$, and imbalance, $\beta$, obtained by each proxy filtering scheme on the Bank Marketing \cite{MORO201422, Dua:2019}, Adult~\cite{Dua:2019}, and Communities and Crime data sets~\cite{Dua:2019, comm1, comm2, comm3, comm4, comm5}, for which we have 5, 4, and 12 sensitive attribute values, respectively. The Marketing data set consists of 45211 labeled samples with 48 non-sensitive attributes and a sensitive attribute of job type. The downstream classification goal is to predict whether a client will subscribe a term deposit based on a phone call marketing campaign of a Portuguese banking institution. The Adult data set consists of 48842 labeled samples with 14 non-sensitive attributes, and we select race as the sensitive attribute. The associated classification task is to determine whether individuals make over \$50K dollars per year. The Communities and Crime data set consists of 1594 samples with 132 non-sensitive features, race as the sensitive group, and the number of violent crimes per population as the prediction task.

For each experiment, we run trials with 20 different seeds, and for each seed, we input a grid of values with increments of 0.1 for the disclosure parameter, $\alpha$, evenly spaced between $0$ and $1$. We then average over the seeds for each $\alpha$ and calculate empirical $95\%$ confidence intervals (which are displayed as the shaded region around each line in the plots). Each data set is split into three parts of sizes 50\%, 30\%, and 20\%. The first is used to train the proxy. The second is used first to test the filtering effects of the proxy out-of-sample and then to train a classification model on to study downstream performance. The third is the set upon which we apply these classifiers trained on filtered and unfiltered data to see how the group-wise accuracy levels are affected. For brevity, we will refer to these three splits as the ``Train'' set, ``Test'' set, and ``Post-Test'' set, respectively. See Appendix~\ref{sec:app_exp} for an analysis of downstream fairness effects induced by our strategy.

On the Adult and Communities and Crime data sets, one run over the grid of $\alpha$ values typically took between 20 minutes and two hours for the $(\alpha,\beta)$ proxy. On the Marketing data set, running one full experiment over the grid of $\alpha$ values took about three hours. The parameter $\gamma$ was set to 0.0001, the maximum height of the proxy tree was set to 15, and the learning process was stopped once the distance between the convex hull of the conditional distribution matrix and the uniform distribution fell below 0.05. As we used publicly available tabular data sets that has already been cleaned, there were no missing values.

Finally, the choice of oracle (the base model class for the $(\alpha,\beta)$ proxy) is heuristic --- as we do not have a true cost-sensitive classification oracle for Algorithm~\ref{alg:split}, we choose two models that allow us to predict the cost of each example and then classify based on the cost's sign. We experiment with a linear threshold function --- the paired regression classifier (PRC) used in \cite{gerrymandering} and defined below -- as well as the XGBoost Regressor model. We found that the PRC was simpler and seemed to perform at least as well as the XGBoost Regressor, so we relegate the analysis for the latter to Appendix~\ref{sec:app_exp}.

\begin{definition}[Paired Regression Classifier \cite{gerrymandering}]
\label{def:PRC}
The paired regression classifier operates as follows: We form two weight vectors, $z^0$ and $z^1$, where $z^k_i$ corresponds to the penalty assigned to sample $i$ in the event that it is labeled $k$. For the correct labeling of $x_i$, the penalty is $0$. For the incorrect labeling, the penalty is the current sample weight of the point, $w_V$. We fit two linear regression models $h^0$ and $h^1$ to predict $z^0$ and $z^1$, respectively, on all samples. Then, given a new point $x$, we calculate $h^0(x)$ and $h^1(x)$ and output $h(x) = \argmin_{k\in\{0,1\}} h^k(x)$.
\end{definition} 

\subsection{Results}
In Figure~\ref{fig:combinedFairness}, on the Communities data set, the $(\alpha,\beta)$ proxy Pareto-dominates the other approaches in sample, while the QP proxies Pareto-dominate SMOTE and the Naive proxies. All methods generalize well. On the Adult data set, the $(\alpha, \beta)$ proxy primarily dominates the remaining approaches in-sample. The generalization performance for all methods, but particularly the $(\alpha,\beta)$ proxy, is weaker on the Adult data set. This is likely because there are slightly more sensitive groups than in the Communities data set, and the acceptance probabilities were sparse. On the Marketing data set, the $(\alpha,\beta)$ and QP Decision Tree proxies exhibit favorable performance in-sample, driving the imbalance to just above zero at higher levels of disclosure. The plot on the test set shows a more modest improvement in balance for all methods. One source of variance in Figure~\ref{fig:combinedFairness} is the generalization performance by the $(\alpha,\beta)$ proxy. We believe this to be due to the size of the smallest proxy group being quite low (especially for the Marketing data set which has 12 sensitive groups). Recall that the generalization gap depends directly on this quantity. There is also nothing in our method to prevent a sparse sampling scheme. Empirically, we found that in cases where generalization results were weak, the acceptance probabilities were nonzero for only a handful of the final proxy groups. Addressing these weaknesses, if possible, could strengthen our approach.

\begin{figure}[!ht]
\captionsetup[subfigure]{justification=centering}
    \centering
    \begin{subfigure}{\textwidth}
    \centering
    \begin{subfigure}{0.49\textwidth}
    \centering
    \includegraphics[width=\textwidth]{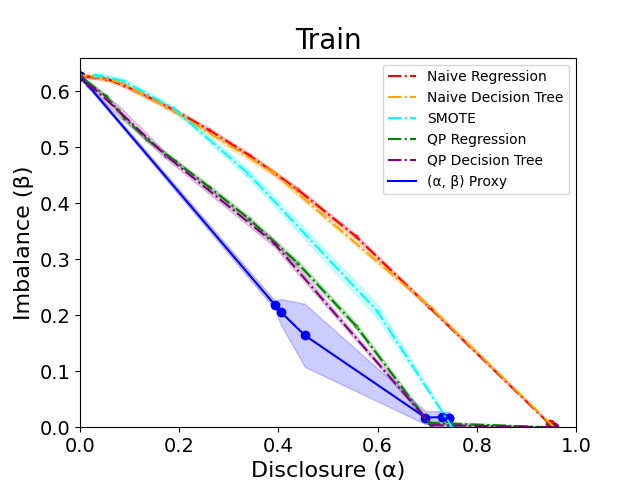}
    \end{subfigure}
    \hfill
    \begin{subfigure}{0.49\textwidth}
    \centering
    \includegraphics[width=\textwidth]{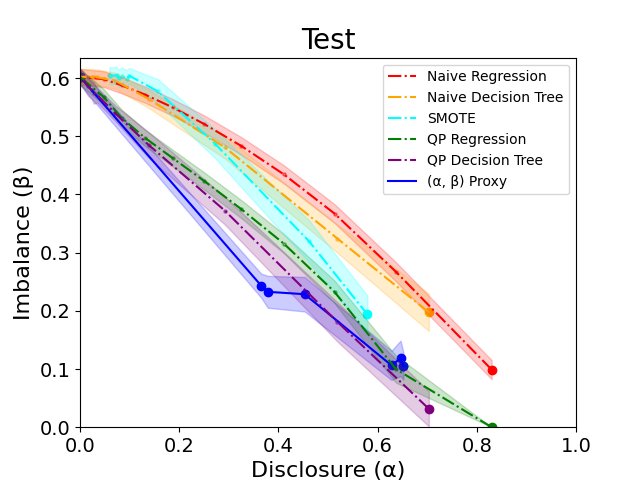}
    \end{subfigure}
    \caption{Communities}
    \end{subfigure}
    \vfill
    \begin{subfigure}{\textwidth}
    \centering
    \begin{subfigure}{0.49\textwidth}
    \centering
    \includegraphics[width=\textwidth]{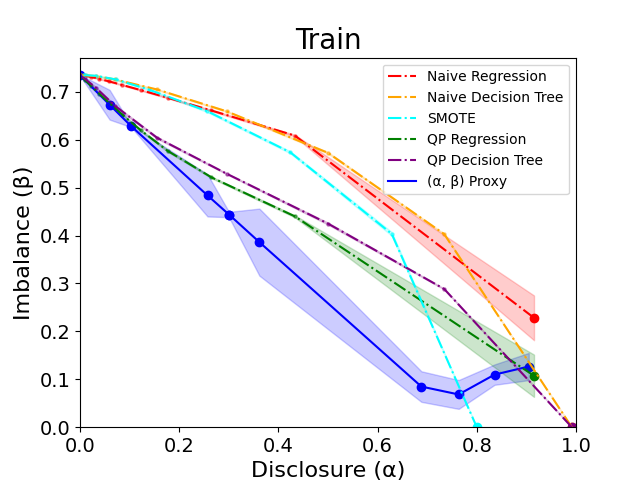}
    \end{subfigure}
    \hfill
    \begin{subfigure}{0.49\textwidth}
    \centering
    \includegraphics[width=\textwidth]{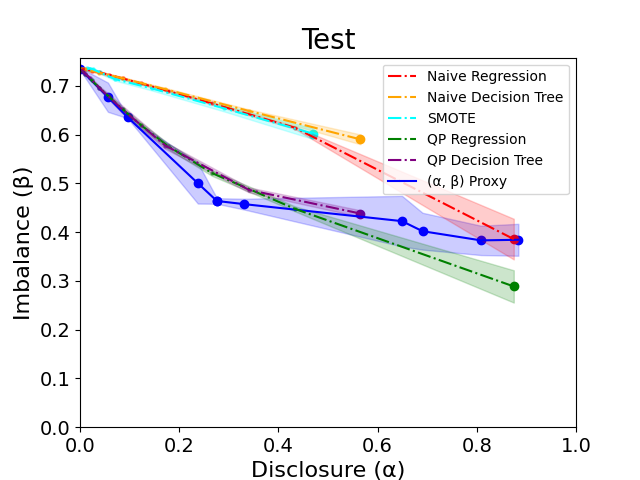}
    \end{subfigure}
    \caption{Adult}
    \end{subfigure}
    \vfill
    \begin{subfigure}{\textwidth}
    \centering
    \begin{subfigure}{0.49\textwidth}
    \centering
    \includegraphics[width=\textwidth]{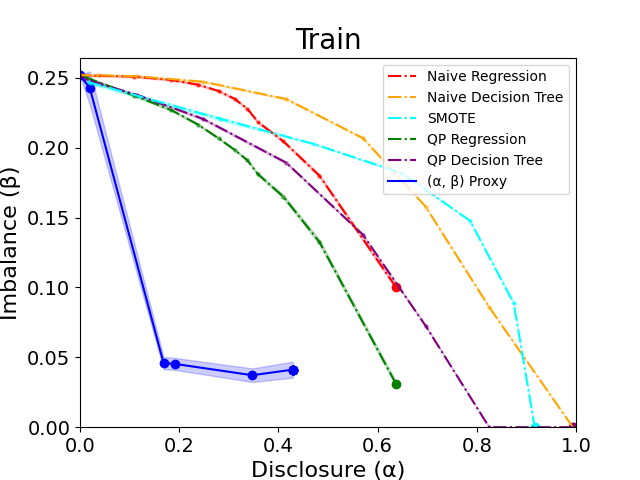}
    \end{subfigure}
    \hfill
    \begin{subfigure}{0.49\textwidth}
    \centering
    \includegraphics[width=\textwidth]{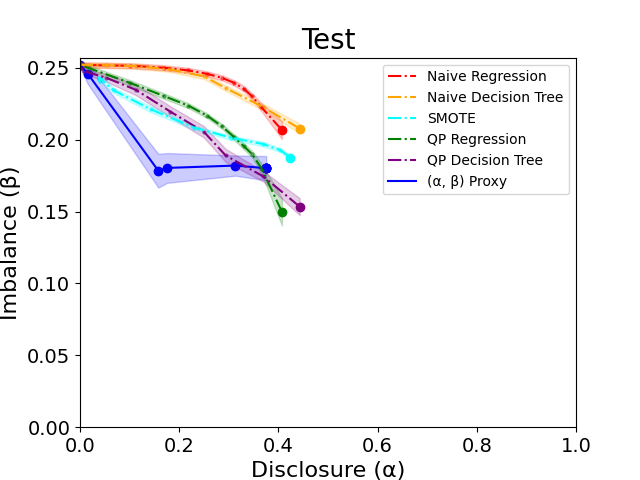}
    \end{subfigure}
    \caption{Marketing}
    \end{subfigure}
    \caption{Trade-off of Disclosure and Balance of Proxies on Communities, Adult, and Marketing}
    \label{fig:combinedFairness}
\end{figure}
\subsection{Discussion and Future Work}\label{sec:discussion}
Our primary conceptual point is that even though the final goal (balance) references the protected attributes, it is a condition on the aggregate composition of the final selected set. Therefore, achieving it does not necessarily require finding a predictor strongly correlated with the protected attribute. We emphasize that while the QP proxies (our secondary contribution) are appealingly simple and provide a range of disclosure levels \textit{after} post-processing, they still involve explicitly training a classifier for the attribute. In contrast, the $(\alpha, \beta)$ proxy (our primary contribution) never involves training a classifier at any step of the process that is more disclosive than a pre-specified threshold. While this does not solve the challenging legal and technical problems associated with proxy use in high-stakes selection processes, it takes a step in this direction by permitting controlled trade-offs between balance and disclosure.

\bibliography{refs}

\appendix
\section{Omitted Proofs}
\label{sec:app_proofs}
\vertexSplitting*
\begin{proof}

We want to find sufficient conditions for $\lVert U-U''\rVert \leq (1-\gamma) \lVert U-U' \rVert $. Let $\phi$ be the angle between the vectors $U-U''$ and $U-U'$. Then $\lVert U - U''\rVert = \lVert U-U'\rVert \cos \phi$. So, we would like to find conditions for which $\cos \phi \leq (1-\gamma)$. By the law of cosines, $\lVert V1-U' \rVert ^2 = \lVert V-U'\rVert^2 + \lVert V1-V \rVert^2 - 2 \lVert V-U'\rVert \lVert V1-V\rVert \cos(90+\theta)$ and
\begin{align*}
\cos(\phi) &= \frac{\lVert V-U'\rVert ^2 + \lVert V1-U'\rVert ^2 - \lVert R(V1) - R(V)\rVert ^2}{2\lVert V-U'\rVert  \lVert V1 -U'\rVert } \\
&=\frac{\lVert V-U'\rVert  - \lVert R(V1) - R(V)\rVert  \cos(90 + \theta)}{\sqrt{\lVert V-U'\rVert ^2 + \lVert R(V1) - R(V)\rVert ^2 -2 \lVert V-U'\rVert  \lVert R(V1) - R(V)\rVert  \cos(90 + \theta)}} \\
&=\frac{\lVert V-U'\rVert  +  \lVert R(V1) - R(V)\rVert  \sin \theta}{ \sqrt{\lVert V-U'\rVert ^2 + \lVert R(V1) - R(V)\rVert ^2 + 2 \lVert V-U'\rVert  \lVert R(V1) - R(V)\rVert  \sin \theta }} \\
\end{align*}

Setting $-(1-\gamma) \leq \cos \phi \leq 1 - \gamma$ and solving for $\sin \theta$, we see that this is satisfied by
\begin{align*}
\sin \theta \in &\frac{(\gamma^2 - 2\gamma) \lVert R(V) - U'\rVert }{\lVert R(V1) - R(V)\rVert } \pm \frac{(1-\gamma)\sqrt{(\gamma^2 - 2\gamma)\lVert R(V) - U'\rVert^2 + \lVert R(V1) - R(V)\rVert ^2}}{\lVert R(V1) - R(V)\rVert }
\end{align*}

To find a set of values for $\cos \theta$ that make the above expression always true, we will consider only $\gamma$ for which the set of values of $\sin \theta$ includes the origin. This is true for $\gamma \in \left[0, 1 - \sqrt{\frac{\lVert V-U'\rVert ^2}{\lVert V-U'\rVert ^2 + \lVert R(V1) - R(V)\rVert ^2}}\right]$. Then, 
\[ \scalemath{0.94}{\cos^2 \theta \geq 1 - \left( \frac{(\gamma^2 - 2\gamma) \lVert R(V) - U'\rVert}{\lVert R(V1) - R(V)\rVert }+ \frac{(1-\gamma)\sqrt{(\gamma^2 - 2\gamma)\lVert R(V) - U'\rVert^2 + \lVert R(V1) - R(V)\rVert ^2}}{\lVert R(V1) - R(V)\rVert }\right)^2} \]

Rearranging, we have

\begin{align*}
&\scalemath{0.85}{\lVert R(V1) - R(V)\rVert^2 \cos^2 \theta \geq}\\&\scalemath{0.85}{\lVert R(V1) - R(V)\rVert^2  - \left((\gamma^2 - 2\gamma) \lVert R(V) - U'\rVert +  (1-\gamma)\sqrt{(\gamma^2 - 2\gamma)\lVert R(V) - U'\rVert^2 + \lVert R(V1) - R(V)\rVert ^2}\right)^2}\\ 
&\scalemath{0.85}{=\left(2\gamma - \gamma^2 \right) \cdot \left(\lVert R(V1) - R(V)\rVert^2- \lVert R(V) - U'\rVert^2 + \right.} \\&\scalemath{0.85}{ \left. 2\lVert R(V) - U'\rVert \left(1-\gamma\right)\sqrt{\left(\gamma^2 - 2\gamma\right)\lVert R(V) - U'\rVert^2 + \lVert R(V1) - R(V)\rVert ^2}\right)}
\end{align*}

Using the fact that $\lVert R(V1) - R(V) \rVert^2 \leq 2$, we upper bound the right-hand side to say that a split satisfying the following condition will guarantee that we decrease the distance from $U$ to the convex hull by $\left(1-\gamma\right)$:
\begin{align*}
&\scalemath{0.91}{\lVert R(V1) - R(V)\rVert \cos \theta \geq} \\ & \scalemath{0.91}{\left(2\gamma - \gamma^2 \right)^{\frac{1}{2}} \left(2 - \lVert R(V) - U'\rVert^2 + 2\lVert R(V) - U'\rVert \left(1-\gamma \right)\sqrt{(\gamma^2 - 2\gamma)\lVert R(V) - U'\rVert^2 + 2}\right)^{\frac{1}{2}} :=f \left(\gamma \right)}
\end{align*}

\end{proof}

\split*

\begin{proof}
We begin by upper bounding the $L_2$ norm of the gradient:

\begin{align*}
    \lVert \nabla_{\lambda} L(\lambda, \tilde{h}_V) \rVert^2 = &\left(\sum_{k=1}^{K}  \sum_{i=1}^{m} \E_{h_V \sim \tilde{h}_V} w(x_i)h_V(x_i) \left( \1_{z_i = k} - r_k - \alpha \right)\right)^2 + \\& \left(\sum_{k=1}^{K} \sum_{i=1}^{m} \E_{h_V \sim \tilde{h}_V} w(x_i)(1-h_V(x_i)) \left( \1_{z_i = k} -r_k - \alpha \right)\right)^2 + \\
    &\left(\sum_{k=1}^{K}  \sum_{i=1}^{m} \E_{h_V \sim \tilde{h}_V} w(x_i)h_V(x_i) \left( r_k - \1_{z_i = k} -\alpha \right)\right)^2 + \\& \left(\sum_{k=1}^{K} \sum_{i=1}^{m} \E_{h_V \sim \tilde{h}_V} w(x_i)(1-h_V(x_i)) \left( r_k - \1_{z_i = k} - \alpha \right)\right)^2 \\
    &\leq 4 K^2 m^2\left(1 + \alpha\right)^2
\end{align*}

We now apply the regret bound for Online Gradient Descent from \cite{zinkevich}. With an appropriate choice of $\eta$ (derived below), we bound the Auditor's average regret over $T$ rounds:

\begin{align*}
    \frac{R_T}{T} &\leq \frac{\sup_{\lambda,\lambda' \in \Lambda}\lVert \lambda- \lambda' \rVert \lVert \nabla_{\lambda} L(\tilde{h},\lambda) \rVert \sqrt{T}} {T} \leq \frac{2 K m \left(1 + \alpha\right) \sup_{\lambda,\lambda' \in \Lambda} \lVert \lambda- \lambda' \rVert}{\sqrt{T}}
\end{align*}

Setting $T \geq \left( \frac{2 K m \left(1+\alpha\right) \sup_{\lambda,\lambda' \in \Lambda} \lVert \lambda- \lambda' \rVert}{\epsilon}\right) ^2$ and $\eta = \frac{\sup_{\lambda,\lambda' \in \Lambda} \lVert \lambda- \lambda' \rVert}{2 K m \left(1 + \alpha\right)\sqrt{T}}$, we have that $\frac{R_T}{T} \leq~\epsilon$. Because the Learner plays a no-regret strategy, we can apply Theorem~\eqref{thm:noregret} to assert that the mixed strategy of the Auditor and Learner together form an $\epsilon$-approximate equilibrium. Next, we must show that an approximate solution to the game corresponds to an approximate solution to Program~\eqref{eqn:expected_program}. We will show this using two cases. In the first case, we consider some $\tilde{h}_V^*$ that is a feasible solution to Program~\eqref{eqn:expected_program} at vertex $V$ and a $\hat{\lambda}$ that is an $\epsilon-$approximate minimax solution to the Lagrangian game specified in the Lagrangian above. Now we will analyze the case in which we have a solution $\tilde{h}_V$ that is an $\epsilon$-approximate solution to the Lagrangian game but is not a feasible solution for Program~\eqref{eqn:expected_program} -- we will show that this is impossible. To illustrate this, assume that we \textit{do} have such a $\tilde{h}_V$. Because it is not a feasible solution for Program~\eqref{eqn:expected_program}, some constraints must be violated. Let $\xi$ be the magnitude of the violated constraint, and let $\lambda$ be such that the dual variable for the violated constraint is set to $\lambda_{max} := \sup_{\lambda,\lambda' \in \Lambda} \lVert \lambda- \lambda' \rVert$. By definition of an $\epsilon$-approximate minimax solution, we know that $L(\hat{\lambda}, \tilde{h}_V) \geq L(\lambda, \tilde{h}_V) \geq \E_{h_V \sim \tilde{h}_V} \sum_{i=1}^{m} w_V h_V(x_i) \sum_{k=1}^{K}  \1_{z_i = k} \left(U_k' - U_k\right) + \lambda_{max} \xi - \epsilon$. Then, 
\begin{align*}
&\E_{h_V \sim \tilde{h}_V} \sum_{i=1}^{m}  w_V  h_V(x_i) 
 \sum_{k=1}^{K}\1_{z_i = k} \left(U_k' - U_k\right) + \lambda_{max} \xi \\ &\leq L(\tilde{h}_V, \hat{\lambda}) + \epsilon \leq L(\tilde{h}_V^*,\hat{\lambda}) + 2\epsilon  \leq \E_{h_V \sim \tilde{h}^*_V} \sum_{i=1}^{m} w_V h_V(x_i) 
 \sum_{k=1}^{K}  \1_{z_i = k} \left(U_k' - U_k\right) + 2\epsilon
\end{align*}

Finally, because $\E_{h_V \sim \tilde{h}^*_V} \sum_{i=1}^{m}  w_V h_V(x_i) \sum_{k=1}^{K} \1_{z_i = k} \left(U_k' - U_k\right) \leq \frac{m(K-1)}{K}$, we have that $\lambda_{max} \xi \geq \frac{m(K-1)}{K} + 2\epsilon$. Therefore, the maximum constraint violation is no more than $\frac{\frac{m(K-1)}{K} + 2\epsilon}{\lambda_{max}}$. Setting $\lambda_{max} = \frac{m(K-1)}{K\epsilon_\alpha} +2$, $\tilde{h}_V$ does not violate any constraint by more than $\epsilon$.
\end{proof}

\generalization*

\begin{proof}
Let $\tilde{A}_{k,j}=\Pr_{(x,y,z) \sim \Ps}[z=k,g(x)=j]$ and $A_{i,j}= \frac{1}{n} \sum_{i=1}^n w_j(x_i)\1_{z_i=k,g(x_i)=j}$. Then, Hoeffding's inequality gives us that, for fixed $k,j, g$

\begin{align*}
    &\Pr_{D \sim \Omega}\left[\left|\frac{1}{n}\sum_{i=1}^n w_j(x_i)\1_{z_i=k,g(x_i)=j} - \mathbb{E}_{(x,y,z) \sim \Ps}[\1_{z_i=k,g(x_i)=j}]\right| > \epsilon \right] \leq 2e^{-2\epsilon^2n} \\
\end{align*}

Recall from Theorem~\ref{thm:rounds}, our decision tree proxy will contain at most $\frac{\ln{\gamma} - \ln{\sqrt{2}}}{\ln{(1- \gamma)}}$ splits, and therefore there will be at most $\frac{\ln{\gamma} - \ln{\sqrt{2}}}{\ln{(1- \gamma)}}$ unique proxy groups. Applying a union bound over all $k,j$ pairs and fixed $g$, we see that 
 \begin{align*}
&\scalemath{0.95}{\Pr_{D \sim \Omega}\left[\cap_{k,j} \1_{\left|\frac{1}{n}\sum_{i=1}^n w_j(x_i)\1_{z_i=k,g(x_i)=j} - \mathbb{E}_{(x,y,z) \sim \Ps}[\1_{z_i=k,g(x_i)=j}]\right|} > \epsilon \right] \leq 2K\frac{\ln{\beta} - \ln{\sqrt{2}}}{\ln{(1- \gamma)}}e^{-2\epsilon^2 n}}
\end{align*}

Again applying a union bound, this time over the model class $g$ -- with VC dimension $d$ -- as well as $k,j$ pairs, we see that for all $k,j,g$, 
\begin{align*}
\label{eq:generalization_joint}
&\scalemath{0.95}{\Pr_{D \sim \Omega}\left[\cap_{k,j} \1_{\left|\frac{1}{n}\sum_{i=1}^n w_j(x_i)\1_{z_i=k,g(x_i)=j} - \mathbb{E}_{(x,y,z) \sim \Ps}[\1_{z_i=k,g(x_i)=j}]\right|} > \epsilon \right] \leq 2dK\frac{\ln{\beta} - \ln{\sqrt{2}}}{\ln{(1- \gamma)}}e^{-2\epsilon^2 n}}
\end{align*}

Setting this to be less than $\frac{\delta}{3}$, we obtain $2dK\frac{\ln{\beta} - \ln{\sqrt{2}}}{\ln{(1- \gamma)}}e^{-2\epsilon^2 n} \leq \frac{\delta}{3}$, which implies $n \geq \frac{1}{2\epsilon^2} \ln {\frac{6dk \left(\ln{\beta} - \ln{\sqrt{2}}\right)}{\delta \ln{(1- \gamma)}}}$. Then, with probability $1-\frac{\delta}{3}$

\begin{align*}
    \lVert \left(A-\tilde{A}\right)\rho \rVert _2 &\leq \sqrt{\sum_{j=1}^{J} \left(\sum_{k=1}^{K}\left(A_{k,j} - \tilde{A}_{k,j}\right) \cdot \rho_j \right)^2} <  \sqrt{\sum_{j=1}^{J} (K \epsilon \rho_j)^2} \leq K\epsilon\sqrt{\frac{\ln{\beta} - \ln{\sqrt{2}}}{\ln{(1- \gamma)}}}
\end{align*}

This bounds the degradation we expect in balance when we apply the proxy out of sample. Next, we consider the degradation in disclosiveness, which will depend on our estimates of $\Pr_{z \sim \Ps_z} [z=k]$ and $\Pr_{z|x \sim \Ps_{z|x}}[z=k|g(x)=j]$. First, we bound the empirical estimate of $\Pr_{z \sim \Ps_z}[z=k]$. Applying Hoeffding's inequality gives \\$\Pr_{D \sim \Omega}\left[\left|\frac{1}{n}\sum_{i=1}^n \1_{z_i=k} - \mathbb{E}_{z \sim \Ps_z}[\1_{z_i=k}]\right| > \epsilon \right] \leq 2e^{-2\epsilon^2 n}$. Applying a union bound over the range of $Z$ gives $\Pr_{D \sim \Omega}\left[\cap_{i=1}^{k}\left|\frac{1}{n}\sum_{l=1}^n \1_{z_i=k} - \mathbb{E}_{z \sim \Ps_z}[\1_{z_i=k}]\right| > \epsilon \right] \leq 2Ke^{-2\epsilon^2 n}$.
Setting this to be less than $\frac{\delta}{3}$ gives us: $2Ke^{-2\epsilon^2 n} \leq \frac{\delta}{3}
\implies n \geq \frac{1}{2\epsilon^2} \ln {\frac{6K}{\delta}}$.

Finally, repeating the exercise for $\Pr_{z|x \sim \Ps_{z|x}}[z|g(x)]$, we have that for fixed $g \in \mathcal{G}$ and $z \in Z$,
\begin{equation*}
\scalemath{0.92}{\Pr_{D \sim \Omega}\left[\left|\sum_{i=1}^n \frac{w_j(x_i)}{\sum_{i=1}^{n} w_j(x_i)}\1_{z_i=k|g(x_i)=j} - \sum_{i=1}^{n} w_j(x_i)\mathbb{E}_{z|x\sim \Ps_{z|x}}[\1_{z_i=k|g(x_i)=j}]\right| > \epsilon \right] \leq 2e^{-2\epsilon^2\sum_{i=1}^{n} w_j(x_i)}}
\end{equation*}

Applying a union bound over $z$, $g$, and the VC dimension of $\mathcal{G}$, and setting the probability to be less than $\frac{\delta}{3}$ gives us $2dK\frac{\ln{\beta} - \ln{\sqrt{2}}}{\ln{(1- \gamma)}}e^{-2\epsilon^2 \sum_{i=1}^{n} w_j(x_i)} \leq \frac{\delta}{3}$, which implies $\sum_{i=1}^{n} w_j(x_i) \geq \frac{1}{2\epsilon^2 }\ln{\frac{6dK \left(\ln{\beta} - \ln{\sqrt{2}}\right)}{\delta\ln{(1- \gamma)}}}$. Then, with probability $1-\delta$ both of our estimates for $\Pr_{z|x \sim \Ps_{z|x}}[z |g(x)]$ and $\Pr_{z \sim \Ps_z}[z]$ must be within $\epsilon$ of the true parameters if we have sample count 
$n \geq \frac{1}{2\epsilon^2} \max \{\ln {\frac{6K}{\delta}},\ln{\frac{6dK\left(\ln{\gamma} - \ln{\sqrt{{2}}}\right)}{\delta\ln{(1- \gamma)}}}\}$. Note that $\max \{\ln {\frac{6K}{\delta}},\ln{\frac{6dK\left(\ln{\gamma} - \ln{\sqrt{2}}\right)}{\delta\ln{(1- \gamma)}}}\} \leq \ln {\frac{6dK \left(\ln{\gamma} - \ln{\sqrt{{2}}}\right)}{\delta \ln{(1- \gamma)}}}$. Then, taking $n \geq \frac{1}{2\epsilon^2}\ln {\frac{6dK \left(\ln{\gamma} - \ln{\sqrt{2}}\right)}{\delta \ln{(1- \gamma)}}}$ suffices. Finally, we can apply our concentration bounds to the expression for disclosure level. If we obtain an $\alpha$-disclosive proxy in-sample, this is equivalent to satisfying, for all $z \in Z$ and $g \in \mathcal{G}$, $\lvert \Pr_{z|x \sim D}[z|g(x)] - \Pr_{z \sim D}[z] \rvert \leq  \alpha \implies \lvert \Pr_{z|x \sim \Ps_{z|x}}[z|g(x)] - \Pr_{z \sim \Ps_z}[z] \rvert \leq  \alpha + 2\epsilon$
\end{proof}

\section{Additional Experimental Details}
\label{sec:app_exp}

In Figure~\ref{fig:xgb}, we show trade-off curves for balance and disclosure when using XGB as the base model. In these plots we also explore a slight relaxation of our $(\alpha,\beta)$ Proxy, in which we remove the constraint $\sum_i q_i =1$ when solving $\min_q \lVert qA -U \rVert_2 \text{ subject to } q_i \geq 0 \; \forall i$. We find that both the original and relaxed version perform similarly. On Communities, there is less stability displayed by the proxies trained with the XGB base model compared to those trained with the PRC base model. On Adult, the proxies trained with XGB as a base model exhibit a smoother trade-off curve. On the Marketing data set, our proxy approach dominates in sample for smaller levels of $\alpha$ but struggles to generalize. 

\begin{figure}[!ht]
\captionsetup[subfigure]{justification=centering}
    \centering
    \begin{subfigure}[b]{0.32\textwidth}
    \centering
    \includegraphics[width=\textwidth]{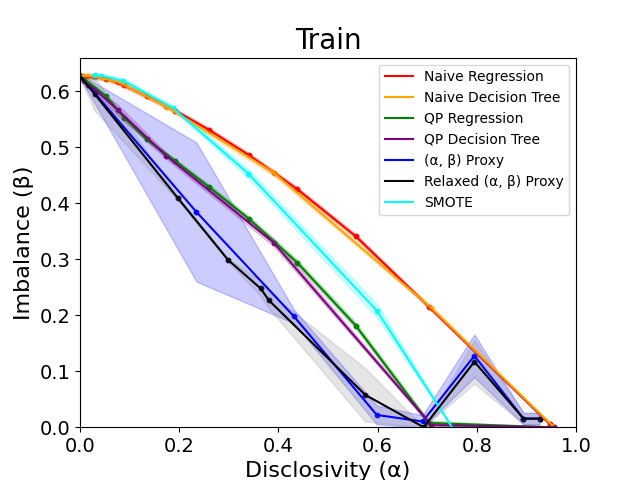}
    \end{subfigure}
    \hfill
    \begin{subfigure}[b]{0.32\textwidth}
    \centering
    \includegraphics[width=\textwidth]{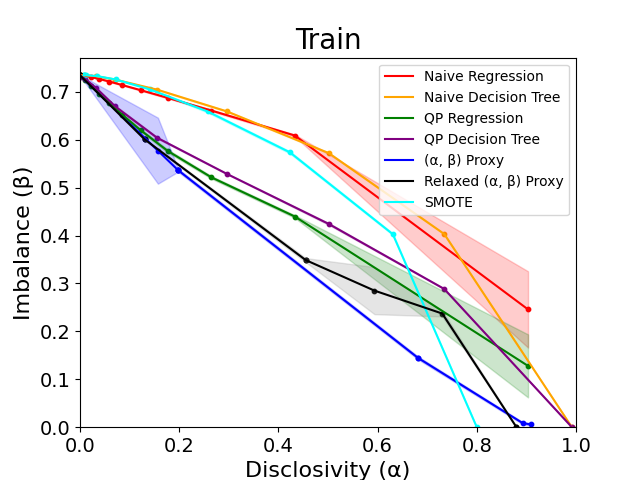}
    \end{subfigure}
    \hfill
    \begin{subfigure}[b]{0.32\textwidth}
    \includegraphics[width=\textwidth]{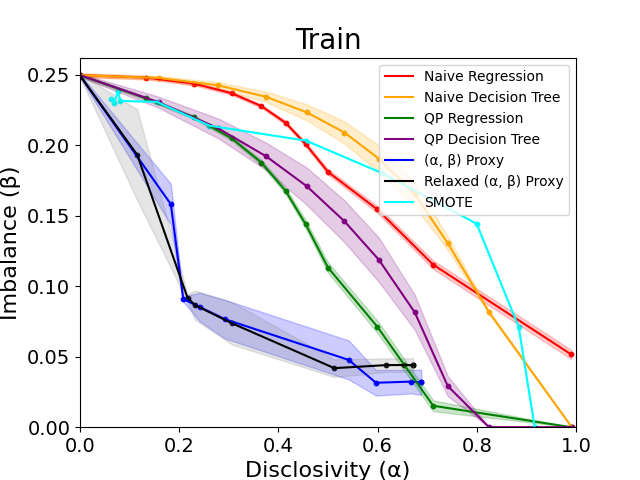}
    \end{subfigure}
    \vfill
    \begin{subfigure}[b]{0.32\textwidth}
    \centering
    \includegraphics[width=\textwidth]{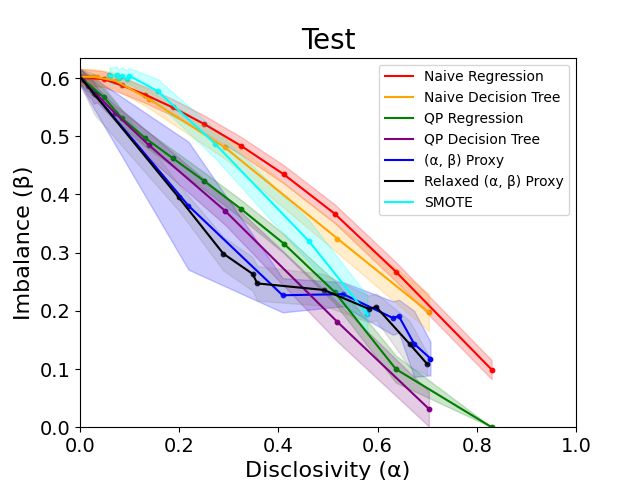}
    \caption{Communities}
    \end{subfigure}
    \hfill
    \begin{subfigure}[b]{0.32\textwidth}
    \centering
    \includegraphics[width=\textwidth]{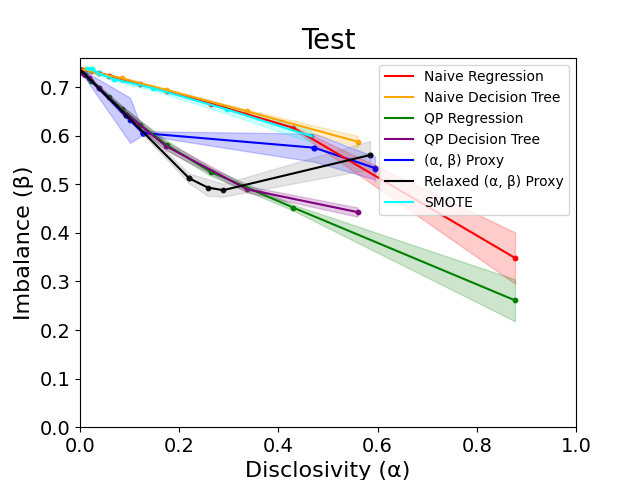}
    \caption{Adult}
    \end{subfigure}
    \hfill
    \begin{subfigure}[b]{0.32\textwidth}
    \includegraphics[width=\textwidth]{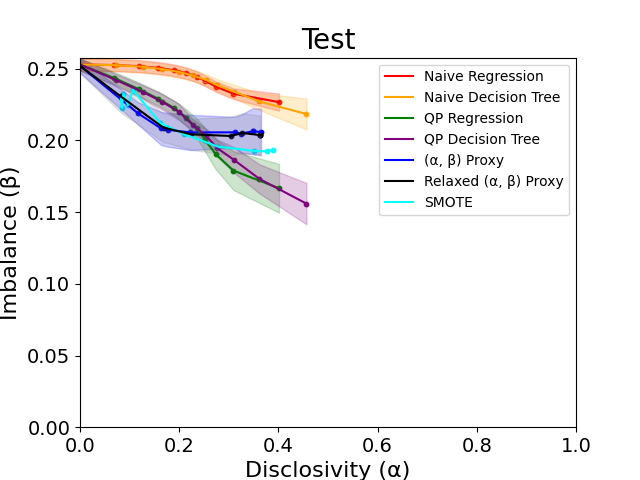}
    \caption{Marketing}
    \end{subfigure}
    \caption{Trade-off of Disclosure and Balance for Proxy Models on the Communities, Adult, and Marketing data sets with XGBoost Base Model}
    \label{fig:xgb}
\end{figure}

Next, we analyze the downstream fairness impact resulting from using our proxy filtering approach to prepare machine learning training datasets. Here, we train a model for the data-specific classification or regression task on an unfiltered sample and a filtered sample of the same size, and we compare the differences in gropup-wise accuracy obtained by each model. We consider the downstream fairness impact of training a model on data that has been filtered by our $(\alpha,\beta)$ proxy function but find our results are inconclusive. While we are able to theoretically guarantee a certain level of balance in the filtered data set, we cannot guarantee that the distribution over features and labels will not be skewed in the filtered set, nor can we guarantee that the distribution over features and labels given sensitive attributes will not be distorted. To test this, we first use an $(\alpha,\beta)$ proxy with a specified $\alpha$ budget to filter the Test set into a balanced sub-sample. Then, we train two model for the dataset specific classification task, one on the filtered data, and the other on a down-sampled version of the original Test set of the same size. We calculate the accuracy of the models on each sensitive group and then plot the \textit{difference} in accuracy between the two models, calculated as the group accuracy on the filtered data minus the group accuracy on the unfiltered data. Thus, positive values indicate an improvement in group accuracy from training on the filtered data, while negative values indicate a decrease. Between the three data sets, we see mixed results, displayed in Figure~\ref{fig:commFairness}. On the Communities data set, we broadly see improvement on lower accuracy groups when using the model trained on the filtered data. However, results from the Adult data set in show a decrease in performance across all groups, and results from the Marketing data set show improvement for one of the least represented groups, but a decrease in performance for most others.

\begin{figure}
\captionsetup[subfigure]{justification=centering}
    \centering
    \begin{subfigure}{\textwidth}
    \centering
    \begin{subfigure}{0.32\textwidth}
    \centering
    \includegraphics[width=\textwidth]{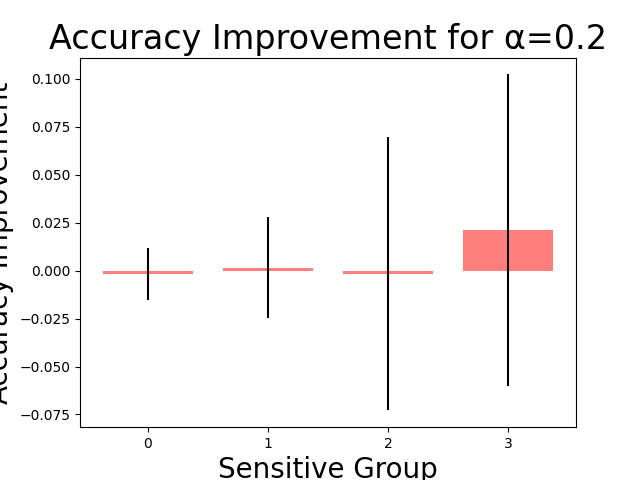}
    \end{subfigure}
    \hfill
    \begin{subfigure}{0.32\textwidth}
    \centering
    \includegraphics[width=\textwidth]{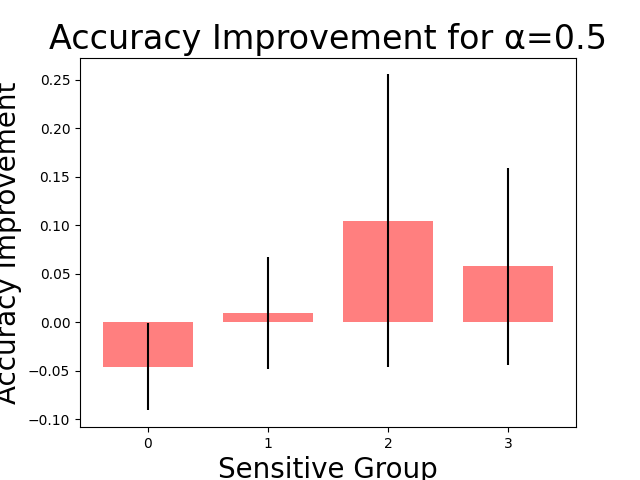}
    \end{subfigure}
    \hfill
    \begin{subfigure}{0.32\textwidth}
    \centering
    \includegraphics[width=\textwidth]{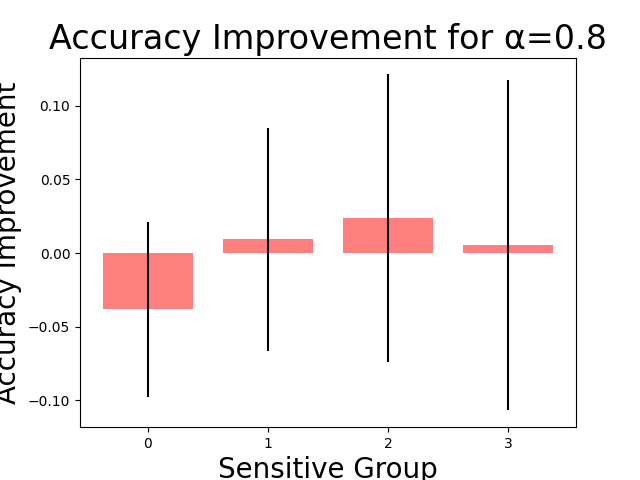}
    \end{subfigure}
    \caption{Communities}
    \end{subfigure}
    \vfill
    \begin{subfigure}{\textwidth}
    \centering
    \begin{subfigure}{0.32\textwidth}
    \centering
    \includegraphics[width=\textwidth]{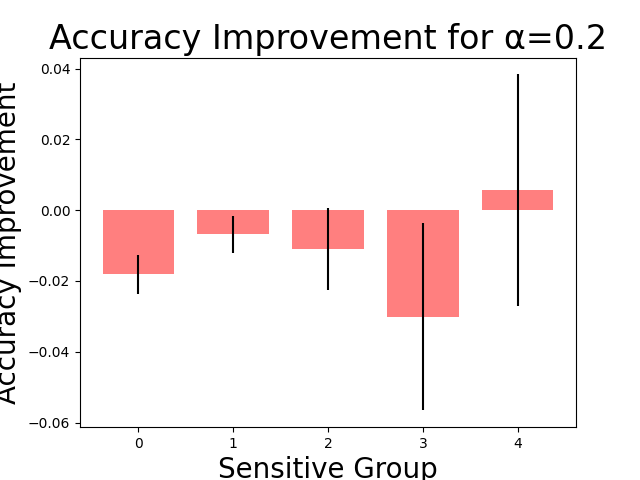}
    \end{subfigure}
    \hfill
    \begin{subfigure}{0.32\textwidth}
    \centering
    \includegraphics[width=\textwidth]{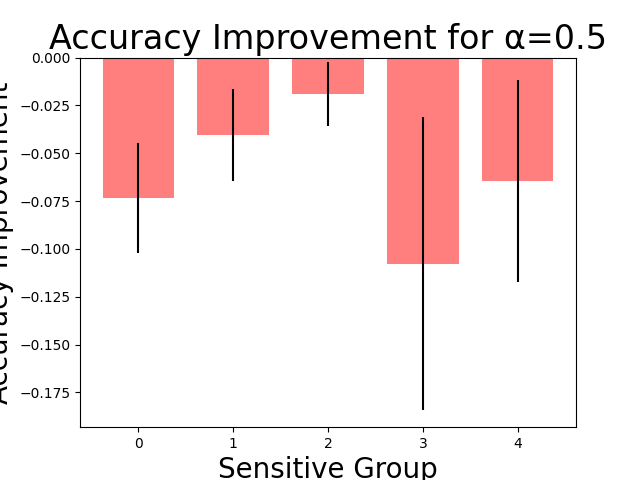}
    \end{subfigure}
    \hfill
    \begin{subfigure}{0.32\textwidth}
    \centering
    \includegraphics[width=\textwidth]{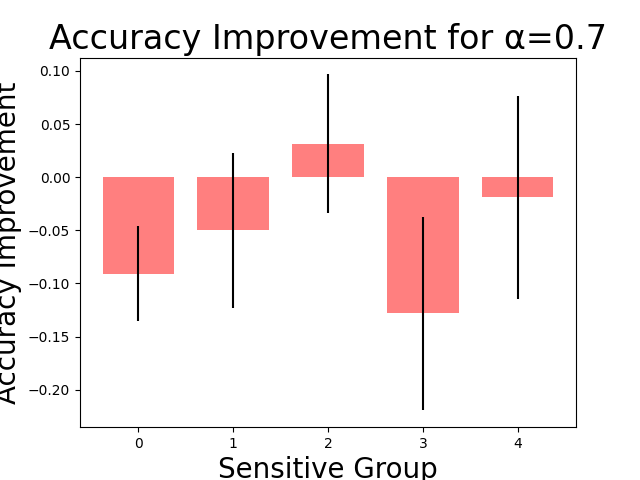}
    \end{subfigure}
    \caption{Adult}
    \end{subfigure}
    \vfill
    \begin{subfigure}{\textwidth}
    \centering
    \begin{subfigure}{0.32\textwidth}
    \centering
    \includegraphics[width=\textwidth]{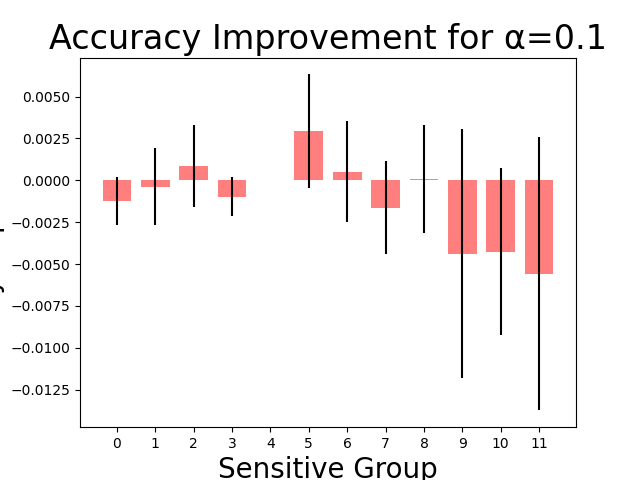}
    \end{subfigure}
    \hfill
    \begin{subfigure}{0.32\textwidth}
    \centering
    \includegraphics[width=\textwidth]{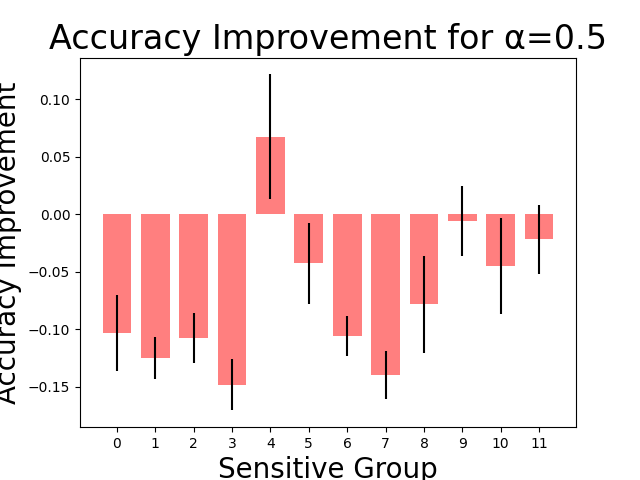}
    \end{subfigure}
    \hfill
    \begin{subfigure}{0.32\textwidth}
    \centering
    \includegraphics[width=\textwidth]{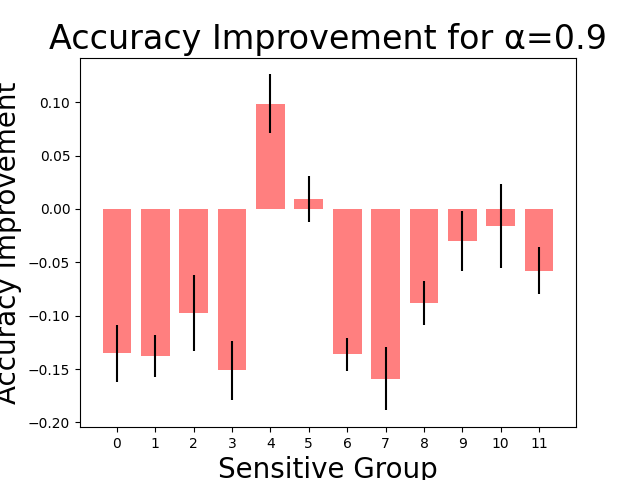}
    \end{subfigure}
    \caption{Marketing}
    \end{subfigure}
    \caption{Difference in accuracy between models trained on filtered and unfiltered data on the Communities, Adult, and Marketing data sets with PRC base model.}
    \label{fig:commFairness}
\end{figure}
\end{document}